\documentclass[12pt]{article}

\usepackage[latin1]{inputenc}
\usepackage{amssymb}
\usepackage{amsmath} 
\usepackage{amsfonts}
\usepackage{rawfonts}
\usepackage{oldlfont}
\usepackage{amsthm}
\usepackage{enumerate}
\usepackage{graphicx}

\newtheorem{lemma}{Lemma}
\newtheorem{th-a}{Theorem}[subsection]
\newtheorem{theorem}{Theorem}
\newtheorem{proposition}{Proposition}
\newtheorem{definition}{Definition}

\newtheorem{cor-a}{Corollary}[subsection]

\newcommand{\eft}[1]{\hat{f}_{#1}}
\newcommand{\<}{\preceq}
\newcommand{\ct}[1]{{\mathrm{crit}}_{\alpha}({#1})}

\def\leq{\leqslant}
\def\Ll{{\mathcal{L}}}
\def\s{\sigma}
\def\ni{\noindent}
\def\X{{\mathcal{X}}}

\def\a{\alpha}
\def\v{\varepsilon}
\def\g{\gamma}
\def\1{1\!{\rm l}}
\def\efm{\hat{f}_m}
\def\ef{\tilde{f}}
\def\Tf{\widetilde{T}}
\def\T{{\mathcal{T}}}
\def\Fm{{{\mathcal{F}}}_m}
\def\Mn{{{\mathcal{M}}}_n}

\def\pen{{{\mathrm{pen}}}_n}
\def\argmin{{\mathop{{\mathrm{argmin}}}}}

\def\P{{{\mathrm{P}}}}
\def\E{{{\mathbb{E}}}}
\def\R{{\mathbb{R}}}
\def\L{{\mathbb{L}}}
\def\F{{\mathcal{F}}}

\def\pn{\par\noindent}
\newcounter{rem}
\newcommand{\Rem}{
  \smallskip
  \stepcounter{rem} 
  \pn {\bf Remark \arabic{rem}. }}

\title{Risk Bounds for CART Classifiers under a Margin Condition}
\author{Servane Gey\footnote{Servane.Gey@parisdescartes.fr, Laboratoire MAP5 - UMR 8145, Universit\'e Paris Descartes, 75270 Paris Cedex 06, France}}

\begin{document}

\maketitle

\begin{abstract}
Non asymptotic risk bounds for Classification And Regression Trees (CART) classifiers are obtained in the binary supervised classification framework under a margin assumption on the joint distribution of the covariates and the labels. These risk bounds are derived conditionally on the construction of the maximal binary tree and allow to prove that the linear penalty used in the CART pruning algorithm is valid under the margin condition.\\ It is also shown that, conditionally on the construction of the maximal tree, the final selection by test sample does not alter dramatically the estimation accuracy of the Bayes classifier. 
\end{abstract}

\ni {\it Keywords}: Classification, CART, Pruning, Margin, Risk Bounds.\\
{\it MSC 2010 classification}: P2010 62G99 62H99

\section{Introduction} \label{sec:intro}

The Classification And Regression Trees (CART) method proposed by
Breiman, Friedman, Olshen and Stone \cite{Brei84} in $1984$ consists in constructing
an efficient procedure that gives a piecewise constant estimator of a
classifier or a regression function from a training sample of
observations. This procedure is based on binary tree-structured partitions and
on a penalized criterion that selects ``good'' tree-structured
estimators among a huge collection of trees. It currently yields some easy-to-interpret and easy-to-compute estimators which
are widely used in many applications in Medicine,
Meteorology, Biology, Pollution or Image Coding  (see \cite{Chou89},
\cite{Wer98} for example). This type of procedure is often performed when the space of explanatory variables is
high-dimensional. Due to its recursive computation, CART needs few computations to provide classifiers, which accelerates the computation time drastically when the number of variables is large. It is now widely used in the genetics framework (see \cite{GeyLeb08} for example), or more generally to reduce variable dimension (see \cite{SauTul06} \cite{Mar06} for example).\\

\ni To construct a decision tree from a training sample of observations, the CART algorithm consists in constructing a deep dyadic recursive tree $T_{max}$ from the observations by minimizing some local impurity function at each step. Then, $T_{max}$ is pruned to obtain an uniquely defined finite sequence of nested trees thanks to a penalized criterion, whose penalty term is of the form 
\begin{eqnarray}
\pen(T) = \alpha \ \frac{|\Tf|}{n}, \label{Eq : PenCART}
\end{eqnarray}
where $\alpha$ is a tuning parameter, $n$ is the number of observations, and $|\Tf|$ is the size of the tree $T$, i.e. the number of leaves (terminal nodes) of $T$. Thus the CART algorithm can be viewed as a model selection procedure, where the collection of models is a collection of random decision trees constructed on the training sample of observations. In its pruning procedure, CART selects a small collection of trees within the whole collection of random trees. 
Then, a final tree belonging to the small collection thus constructed is selected either by cross-validation or by test sample. The present paper focuses on the test sample method. \\

\ni CART differs from the procedure proposed by Blanchard {\it et al.} \cite{blaSchRozMul07} in that the first large tree is constructed locally, and not in a global way by minimizing some loss function on the whole sample. For further results on the construction of the deep tree $T_{max}$, we refer to Nobel \cite{Nob97, Nob02}, and Nobel and Olshen \cite{NobOls96} about Recursive Partitioning.\\
In this paper, our concern is the pruning step which entails the choice of the penalty function \eqref{Eq : PenCART}: the linearity of the penalty term is fundamental to ensure that the whole information is kept in the obtained sequence. Gey {\it et al.} \cite{GeyNed05} addressed this question in the regression framework. Following this previous work, the present paper aims at validating the choice of the penalty in the two class classification framework. Former results on binary classification (see Nobel \cite{Nob02}, or Scott {\it et al.} \cite{ScoNow06} in the image context) provide optimal trees in terms of risk conditionally on the construction of the first large dyadic tree $T_{max}$. These trees are obtained by penalizing the empirical misclassification rate with a penalty term of the form
\begin{eqnarray} \label{eq:PenNob}
\pen(T) & = & \alpha \ \sqrt{\frac{|\Tf|\log n}{n}}. 
\end{eqnarray}
Unfortunately, as discussed by Scott in \cite{Sco05}, the pruning algorithm computed with non-linear penalties is computationally slower than the one using linear penalties, and provides subtrees that are not necessarily unique nor nested. \\ 

\ni The latter results are obtained without making any assumption on the joint distribution $\P$ of the variables. By adding an assumption on $\P$, we exhibit non-asymptotic conditional risk bounds for the tree chosen thanks to the usual CART algorithm as described above. These risk bounds improve those obtained in previous papers (see \cite{Nob02}, \cite{Sco05}, \cite{ScoNow06} for instance); they validate the form of the penalty \eqref{Eq : PenCART} used in the pruning step, and show that the impact of the selection via test sample is conveniently controlled.\\

\ni In this paper, we leave aside the problem of consistency of CART. CART is known to be non-consistent in many cases. Some results and conditions to obtain consistency can be found in Devroye {\it et al.} \cite{DevGyoLug96}. Furthermore, Section \ref{sec:risk} briefly presents consistent results for CART based on the risk bounds obtained.

\bigskip

The outline is the following. Section \ref{sec:cart} gives the general framework of binary classification, an overview of the CART procedure, and introduces the methods and notations used in the following sections.
Section \ref{sec:risk} presents the main theoretical results for classification trees: Theorem \ref{thfin} bears on the whole procedure, while Propositions \ref{pruM1}, \ref{pruM2} concern the pruning procedure and Proposition \ref{FS} concerns the final step. Section \ref{sec:ccl} offers propects about the margin effect on classification trees. Proofs are gathered in Section \ref{sec:append}.

\section{Classification with CART} \label{sec:cart}

\subsection{Binary classification} \label{subsec:classif}

The CART method is used in the following general classification framework. Suppose one observes a sample of $N$ independent copies \\ $(X_1,Y_1), \ldots, (X_N,Y_N)$ of the random variable $(X,Y)$, where the
 explanatory variable $X$ takes values in a measurable space $\X$ and is
associated with a label $Y$ taking values in $\{0,1\}$. A  classifier is then any function $f$ mapping $\X$ into $\{0,1\}$. Its quality is measured by its misclassification rate $$\P(f(X)\neq Y),$$ where
$\P$ denotes the joint distribution of $(X,Y)$. If $\P$ were known, the problem
of finding an optimal classifier minimizing the misclassification rate would
be easily solved by considering the Bayes classifier $f^*$ defined for every
$x\in \X$ by
\begin{equation} \label{bayes}
f^*(x)=\1_{\eta(x)\geqslant 1/2},
\end{equation}
where $\eta(x)$ is the
conditional expectation of $Y$ given $X=x$, that is
\begin{equation} \label{eta}
\eta(x)=\P\left[Y=1 \ | \ X=x\right],
\end{equation}
and $\1$ denotes the indicator function. As $\P$ is unknown, the goal is to construct from the sample
$\{(X_1,Y_1),\ldots,(X_N,Y_N)\}$ a classifier $\ef$ that is as close as
possible to $f^*$ in the following sense: since $f^*$ minimizes the
misclassification rate, $\ef$ will be chosen in such a way that its
misclassification rate is as close as possible to the misclassification rate
of $f^*$, i.e. in such a way that the loss
\begin{equation} \label{loss}
l(f^*,\ef)=\P(\ef(X)\neq Y)-\P(f^*(X)\neq Y)
\end{equation}
is as small as possible. Then, the quality of $\ef$ will be measured by its risk, i.e, the expectation with respect to the sample distribution
\begin{eqnarray} \label{risk}
\E[l(f^*,\ef)].
\end{eqnarray}

\ni Numerous papers have dealt with the issue of predicting a label from an input $x\in \X$ via the construction of a classifier (see for example \cite{AizBraRoz70}, \cite{VapCher74}, \cite{DevGyoLug96}, \cite{ShaFreBarSun98}, \cite{HasTibFri01}). There is a large collection of methods coming both from computational and statistical areas and based on learning a classifier from a learning sample, where the inputs and labels are known. For a non exhaustive yet extensive bibliography on this subject, we refer to Boucheron {\it et al.} \cite{BouBouLug05}. \\ The classifiers considered in the present paper are classical empirical risk minimizers (also referred to as ERM classifiers), where the empirical misclassification rate on a sample ${\mathcal{E}}$ of size $m$ is defined, for any classifier $f$, by
\begin{eqnarray} \label{contrast}
\P_{m}(f)=\frac{1}{m}\sum_{(X_i,Y_i)\in {\mathcal{E}}}\1_{Y_i\neq f(X_i)}.
\end{eqnarray} 
\ni The ERM classifier $\ef$ studied here is computed by classical hold out: the sample $\left\{(X_1,Y_1); \ldots; (X_N,Y_N)\right\}$ of the random variable $(X,Y)\in \X\times \{0,1\}$ is split in two independent subsamples: a learning sample $\Ll$ of size $n_l$ and a test sample $\T$ of size $n_t$, with $n_l+n_t=N$. A collection of ERM classifiers is computed by minimizing $\P_{n_l}$ (equation \eqref{contrast} with ${\mathcal{E}}=\Ll$) on a collection of models, and the final classifier $\ef$ is computed by minimizing $\P_{n_t}$ (equation \eqref{contrast} with ${\mathcal{E}}=\T$) over the collection obtained in that way.\\

\subsection{CART classifiers} \label{subsec:cart}

The CART algorithm provides piecewise constant classifiers represented by binary decision trees. An example of the latter is given in Figure \ref{tree} for a couple of covariates $(X^1,X^2)$ belonging to $\X=[0;1]^2$.

\begin{figure}[ht!]
\begin{center}
~\hspace{-1cm} \includegraphics[width=7cm, height=6cm]{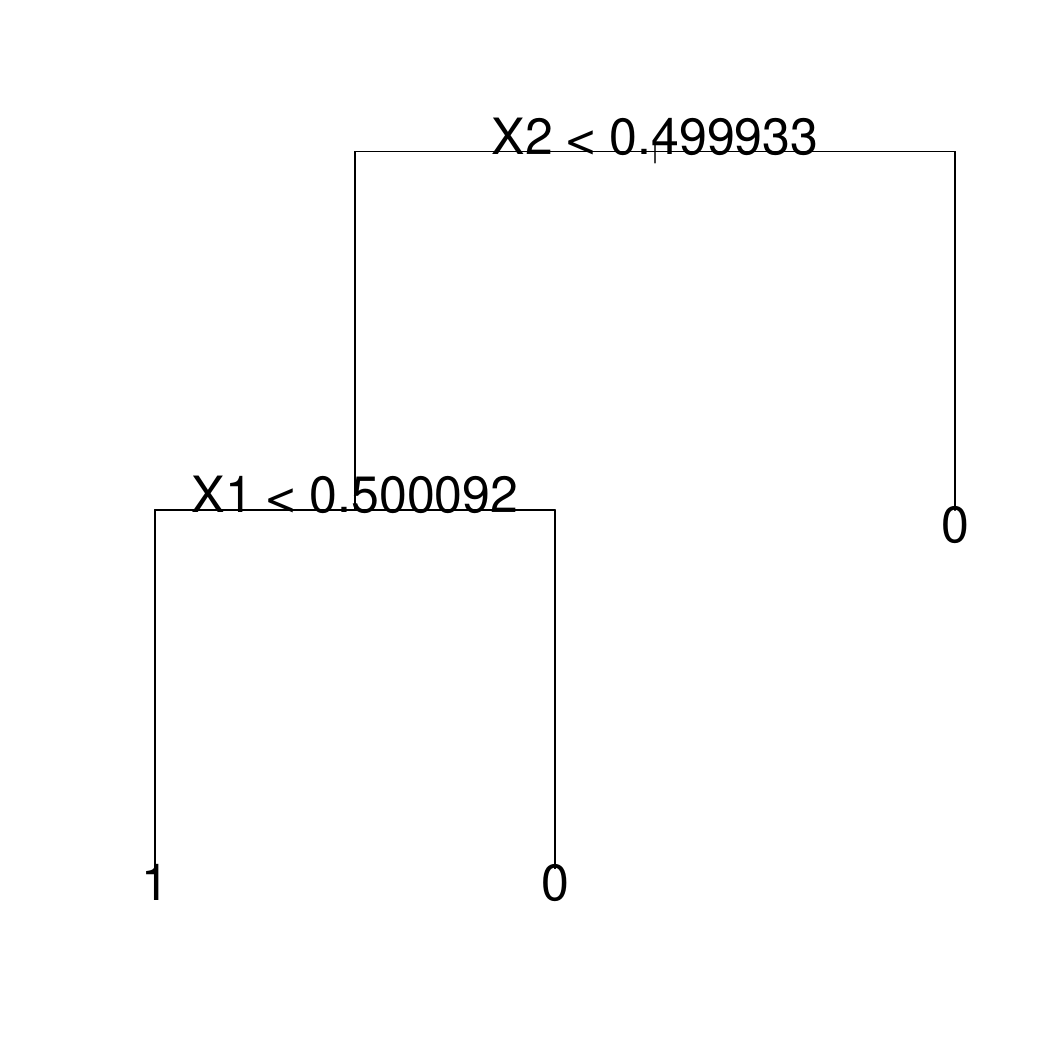}
\includegraphics[width=7cm, height=6cm]{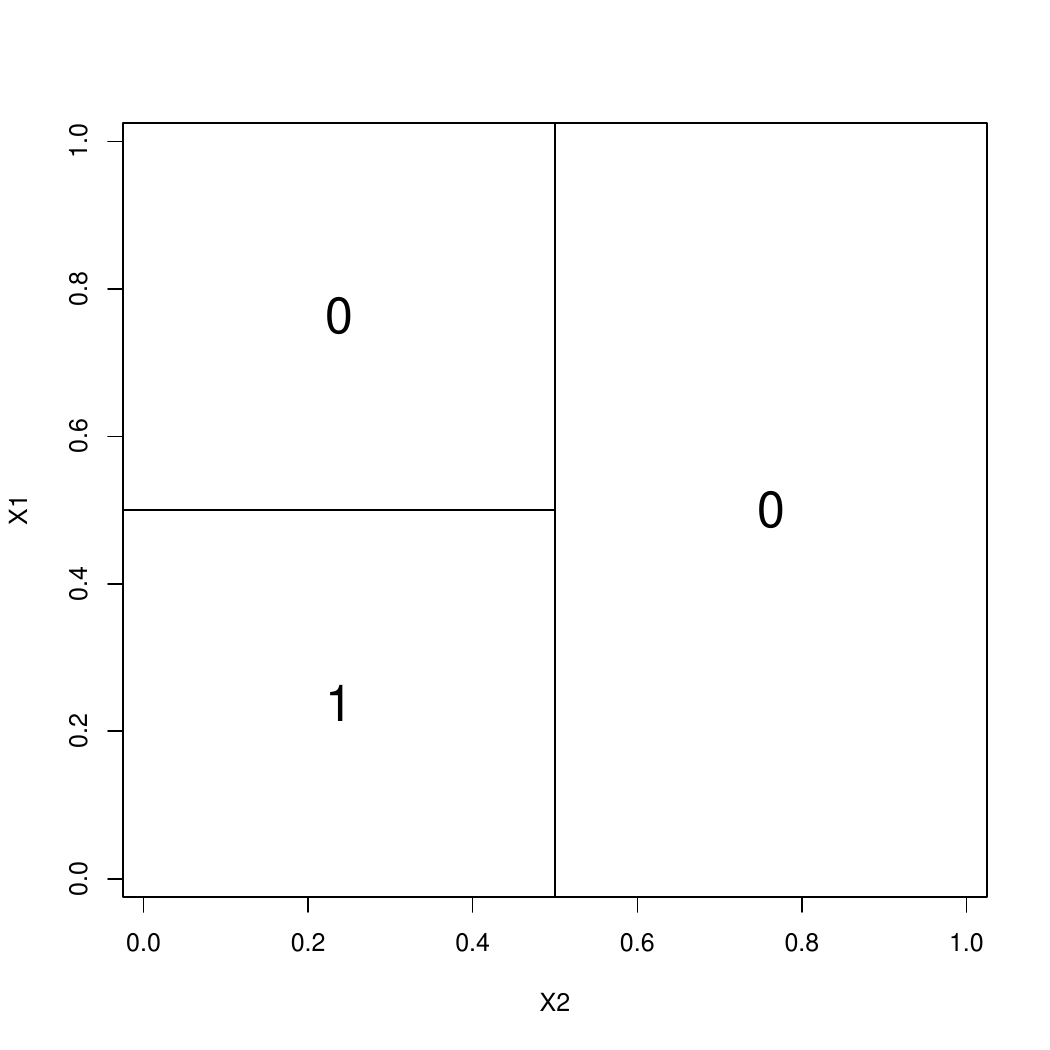}
\end{center}
\caption{Decision tree example (left) and its associated partition (right).} \label{tree}
\end{figure}

\ni The tree on the left hand side of Figure \ref{tree} defines the partition of $\X$ represented on the right hand side  of Figure \ref{tree}: each question asked on an internal node relates to a split in $\X$. If the answer to the question is positive, go to the left child node, if not, go to the right child node. 
Hence the first question corresponds to a two-part partition of the covariate space. Then, each part is split into two subparts, and so on. 
Thus $\X$ is associated to the so called {\it root} of the tree, and the final partition is associated to the terminal nodes, also called {\it leaves}, of the tree. Hence each node of the tree represents a subset of the covariates space defined by the successive questions. The final partition is given by the leaves of the tree. 
Finally, a predictive value for the dependent variable is associated to each leaf. Thus, if $\Tf$ denotes the set of leaves of a decision tree $T$, the classifier $f_T:\X \mapsto \{0;1\}$ defined on $\Tf$ can be written as
\begin{eqnarray}
 f_T=\sum_{t\in \Tf}a_t\1_t, \label{equ:classarbre} 
 \end{eqnarray}
where $a_t\in \{0;1\}$ and $\1_t(x)=1$ if $x$ falls in the leaf $t$, $\1_t(x)=0$ otherwise.\\

\subsection{The CART algorithm} \label{subsec:algo} 

CART is based on a recursive partitioning using a class ${\mathcal{S}}$ of subsets of $\X$ which determines the question to be asked at each internal node of the tree. Below, we consider general classes ${\mathcal{S}}$ with
finite Vapnik-Chervonenkis dimension, henceforth referred to as VC-dimension (for a
complete overview of the VC-dimension see \cite{Vap98}). Let us notice that, theoretically, CART can be performed with any kind of split class ${\mathcal{S}}$, but, in practice, the more frequently used class is that of half spaces of $\X$ with axis-parallel  frontiers (which corresponds to axis-parallel cuts) for computational reasons.\\

\ni To begin with, a collection of CART classifiers is constructed by using learning sample $\Ll$. This collection is computed in two steps, called the growing algorithm and the pruning algorithm. The growing algorithm allows to construct a maximal binary tree $T_{max}$ from the data by recursive partitioning, and then the pruning algorithm allows to select a finite collection of subtrees of $T_{max}$.\\
\ni Since our main interest in this paper is the pruning algorithm, we skip the growing algorithm (for more details about the growing algorithm, see \cite{Brei84}). Just notice that the maximal tree $T_{max}$ is constructed from the learning sample in such a way that, at the end of the algorithm, its leaves are pure, i.e, contain only observations having the same label.\\

\ni Then, to avoid overfitting, a decision tree  having good predictive performance has to be selected among all possible subtrees pruned from $T_{max}$. Let us recall that a pruned subtree of $T_{max}$ is defined as any binary subtree of $T_{max}$ having the same root (denoted $t_1$) as $T_{max}$. As mentioned in \cite{Brei84}, looking at the whole family of subtrees pruned from $T_{max}$ is an NP-hard problem. Then, a good alternative to the exhaustive search is the pruning algorithm, which is computed as follows.\\

\ni First, let us introduce some notations:
\begin{enumerate}[(i)]
\item For a tree $T$, $t$ is the general notation for a node of $T$ and, if $t$ is an internal node, $T_t$ denotes the {\it branch} of $T$ issued from $t$, that is the subtree of $T$ whose root is $t$. 
\item For a tree $T$, $\Tf$ denotes the set of its leaves and $|\Tf|$ the
  cardinality of $\Tf$.
\item Take two trees $T_1$ and $T_2$. Then, if $T_1$ is a pruned subtree of
  $T_2$, write $T_1\<T_2$.
\end{enumerate}
\ni In the meantime, let us denote by $n$ the size of the sample used to prune $T_{max}$; in the methods detailed below, we will see that, in any case, $n\leq n_l<N$, where $n_l$ is the size of the learning sample $\Ll$.\\ Second, let us notice that, given a tree $T$ and $\F_T$ the set of classifiers defined on $\Tf$ as defined by \eqref{equ:classarbre}, the ERM classifier on $\F_T$ is
\begin{eqnarray*}
\eft{T} & = & \argmin_{f\in \F_T}\P_n(f)\\
        & = & \sum_{t\in \Tf} \hat{y}_t\1_t,
\end{eqnarray*}
\ni where $\P_n$ is the empirical misclassification rate defined by
\eqref{contrast}, and $\hat{y}_t \in \{0;1\}$ is the majority vote inside the
leaf $t$. Thus, if $t$ is an internal node of $T$, $\eft{|T_t}$ denotes the restriction of $\eft{T}$ to the sub-partition associated with the leaves of the branch $T_t$, and $\P_n(t)=n^{-1}\sum_{\{X_i\in t\}}\1_{\hat{y}_t\neq Y_i}$ denotes the weighted misclassification rate inside the node $t$.\\
\ni Third, given any subtree $T\< T_{max}$ and $\a>0$, one defines 
\begin{eqnarray} \label{cartcrit}
\ct{T}=\P_{n}(\eft{T})+\a\frac{|\Tf|}{n}.
\end{eqnarray}
the penalized criterion of $T$ for the so called temperature $\a$, and $T_{\a}$ the subtree of $T_{max}$ satisfying:
  \begin{enumerate}[\hspace{0,3cm} (i)]
  \item $T_{\a}=\argmin_{T\< T_{max}}\ct{T}$,
  \item if $\ct{T}=\ct{T_{\a}}$, then $T_{\a}\< T$.
  \end{enumerate}

\ni Thus $T_{\a}$ is the smallest minimizing subtree for the temperature
$\a$. The existence and the unicity of $T_{\a}$ are proved in \cite[pp 284-290]{Brei84}. \\ The pruning algorithm's principle is to raise temperature $\a$, and to record the corresponding $T_{\a}$. The algorithm is summarized in Table \ref{tab:pruning} (see \cite[pp 59-92]{Brei84} for a complete overview). \\

\begin{table}[ht!]
\begin{center}
\begin{tabular}{| r | l |}
\hline
  & {\bf Pruning algorithm}\\
\hline
 & \\
 {\bf Input} & Binary decision tree $T_{max}$.\\
                 & \\
{\bf Initialization} & $\a_1=0$, $T_1=T_{\a_1}=\argmin_{T\< T_{max}}\P_n(\eft{T})$.\\
                           & {\tt Set} $T=T_1$ and $k=1$.\\
                           & \\
{\bf Iteration}   & {\tt While $|\Tf|>1$}, \\
                       & \hspace{0.5cm} {\tt Compute}\\
                       & \hspace{0.8cm} $\a_{k+1} = \displaystyle{\min_{\{t \ \mbox{internal node of} \ T\}}\frac{\P_n(t)-\P_n(\eft{|T_t})}{|\Tf_t|-1}}$.\\
                       & \hspace{0.5cm} {\tt Prune} all branches $T_t$ of $T$ verifying \\
                       & \hspace{0.8cm} $\P_n(\eft{|T_t})+\a_{k+1}|\Tf_t|=\P_n(t)+\a_{k+1}$\\
                       & \hspace{0.5cm} {\tt Set} $T_{k+1}$ the pruned subtree obtained in that way.\\
                       & \hspace{0.5cm} {\tt Set}  $T=T_{k+1}$ and $k = k+1$.\\
                       & \\
{\bf Output}      & Trees $T_1\succ \ldots  \succ T_K=\{t_1\}$,\\
                        & Temperatures $0=\a_1<\ldots< \a_K$.\\
                       & \\
\hline
\end{tabular}
\caption{CART pruning algorithm.} \label{tab:pruning}
\end{center}
\end{table}

\Rem
\begin{enumerate}[1) ]
\item $T_1$ is the smallest subtree for temperature $0$, so it is not necessarily equal to $T_{max}$. 
\item $T_{max}$ are $T_1$ are constructed in such a way that, for all $T\< T_1$ and all internal node $t$ of $T$, $\P_n(t) > \P_n(\eft{|T_t})$; hence, $\a_k>0$ for all $k>1$.
\item The pruning algorithm is designed to catch, at each iteration $k$, the minimal temperature $\a_{k+1}$ for which the overall energy is kept, that is for which ${\mathrm{crit}}_{\a_{k+1}}(T_{k+1})={\mathrm{crit}}_{\a_{k+1}}(T_k)$. This property results directly from the linearity of the penalty used in criterion \eqref{cartcrit}.
\end{enumerate}

\ni Finally, the selection of a tree among the sequence $(T_k)_{1\leq k\leq K}$ is made by using test sample $\T$: choose $\hat{k}$ as
\begin{eqnarray} \label{estim}
 \hat{k}=\underset{\{1\leq k\leq K\}}{\argmin}\left[\P_{n_t}(\eft{T_k})\right],
\end{eqnarray}
\ni where $\P_{n_t}$ is the empirical misclassification rate on $\T$ as defined by \eqref{contrast}.
Then, the final CART classifier is $$\ef = \eft{T_{\hat{k}}}.$$

\subsection{Properties of the pruned subtrees sequence}

It may be easily seen that the computational complexity of the pruning algorithm is linear with respect to the number of nodes of $T_{max}$. Hence, the pruning algorithm is interesting in two ways:
\begin{enumerate}[1)]
\item It reduces drastically the computational complexity of the exhaustive search from ${\mathcal{O}}(n^2)$ to ${\mathcal{O}}(n\log{n})$ (see \cite{Sco05} for instance),
\item It provides a small collection of trees that can be easily evaluated on $\T$.
\end{enumerate}

\ni Thus, to ensure that the CART algorithm provides good classifiers, it is important to verify that
\begin{itemize}
\item pruning is like looking at the entire family of pruned subtrees according to penalized criterion \eqref{cartcrit},
\item pruning provides trees having good performance in term of risk conditionally on the growing algorithm,
\item using a test sample does not alter too much the performance of the tree thus selected.
\end{itemize}

\ni The first point has already been established by Breiman {\it et al.} \cite{Brei84}:
\begin{th-a}[Breiman, Friedman, Olshen, Stone \cite{Brei84}] \label{thbr}
~\\
For all $k \in \{1,\ldots,K\}$, $T_k=T_{\a_k}$ and, for all $\a>0$, there exists $k \in \{1,\ldots,K\}$ satisfying $T_{\a}=T_k$.
\end{th-a}
\ni Theorem \ref{thbr} ensures that
\begin{enumerate}[1)] 
\item the trees of the sequence are unique and minimize penalized criterion \eqref{cartcrit} for known temperatures,
\item whatever the choice of the temperature $\a$ used in the penalized criterion \eqref{cartcrit}, $T_{\a}$ belongs to the sequence. 
\end{enumerate}
\ni Thus, the definition of $T_{\a}$ leads to an infinite collection of trees over all real $\a$, but only finitely many trees are possible according to criterion \eqref{cartcrit}. \\

\ni To the best of our knowledge, the fact that the classifiers provided by CART perform well in terms of conditional risk remains to be seen. To proceed, two methods are applied to construct the sequence $(T_k)_{1\leq k\leq K}$. These methods, as well as the general notations and assumptions refered to in this paper, are presented in the next section.

\section{Methods, Notations and Assumptions} \label{sec:method}

\subsection{Methods and notations} \label{subsec:notations}

\ni For a given tree $T$, $\F_T$ will denote the set of classifiers defined on the partition given by the leaves of $T$, that is 
\begin{equation} \label{eq:model}
\F_T=\left\{\sum_{t\in \Tf}a_t\1_t \ ; \ (a_t)\in \{0,1\}^{|\Tf|}\right\}.
\end{equation}
Thus $\eft{T}=\sum_{t\in \Tf} \hat{y}_t\1_t$ is the ERM classifier on $\F_T$. \\

\ni The two different methods applied in the CART pruning algorithm are:
\begin{enumerate}[\hspace{0,5cm} M1:]
\item $\Ll$ is split in two independent parts $\Ll_1$ and $\Ll_2$ containing respectively $n_1$ and $n_2$ observations, with $n_1+n_2=n_l=N-n_t$. Hence $T_{max}$ is constructed using $\Ll_1$, then pruned using
  $\Ll_2$. This method is applied in Gelfand {\it et al.} \cite{GelRavDel91} for instance.
\item $T_{max}$ is constructed and pruned using sample $\Ll$ entirely. This is the most commonly used method in the CART literature and its applications.
\end{enumerate}
\ni Note that 
a penalty is needed in both methods in order to reduce the number of candidate
tree-structured models contained in $T_{max}$. Indeed, if one does not
penalize, the number of models to be considered grows exponentially with $N$ (see \cite{Brei84}). So making a selection by using a test sample without penalizing requires visiting all the models. As mentioned above, looking for the best model in the collection of all subtrees pruned from the maximal one becomes explosive. Hence pruning allows to reduce significantly the number of trees taken into account.
\ni With both M1 and M2 methods, $\T$ is used to select a tree among the pruned sequence. Let us mention that $\T$ usually represents 10\% of the data and is randomly taken in the original sample, except if the design is fixed. In that case one takes, for example,
one observation out of ten to obtain the test sample. In a similar way, for the M1 method $\Ll_1$ and $\Ll_2$ are taken randomly in $\Ll$, except if the design is fixed, in which case one takes one observation out of two for instance.\\

\ni Methods M1 and M2 involve different treatments for the risks of the CART classifiers thus obtained. Indeed, by conditioning with respect to the sample used to perform the growing algorithm, $T_{max}$ becomes deterministic with M1, while it implies random models depending on the sample used to prune $T_{max}$ with M2. In the latter case, union bounds on the family of all possible trees that can be constructed on the grid $\{X_i \ ; \ (X_i,Y_i)\in \Ll\}$ are used to obtain risk bounds. This allows to obtain risk bounds only conditionally on this grid instead of conditionally on the grid and the labels. To simplify the notations, we define the loss and the $\L^2$ distance corresponding with either method M1 or M2.
\begin{definition} \label{def:perte}
The loss of a classifier $f$ is defined by $\lambda(f^*,f)$, and is computed as follows:\\
\ni $(i)$ if $\ef$ is constructed via M1, $\lambda(f^*,f):= l(f^*,f)$, with $l$ defined by \eqref{loss}.\\
\ni $(ii)$ if $\ef$ is constructed via M2,
\begin{eqnarray*}
\lambda(f^*,f) & := & \E\left[\P_{n_l}(f)-\P_{n_l}(f^*) \ | \ X_i \ ; \ (X_i,Y_i)\in \Ll \right]\\
                      & = & \frac{1}{n_l}\sum_{\{X_i \ ; \ (X_i,Y_i)\in \Ll\}}\left|2\eta(X_i)-1\right|\1_{f(X_i)\neq f^*(X_i)}, 
\end{eqnarray*}
where $\P_{n_l}$ is the empirical misclassification rate on $\Ll$ defined by \eqref{contrast}, and $\eta$ is defined by \eqref{eta}. 
\end{definition}
\ni Since $l(f^*,f)=\E\left[\left|2\eta(X)-1\right|\1_{f(X)\neq f^*(X)}\right]$ for all classifier $f$ (see \cite{DevGyoLug96} for instance), $\lambda$ is just the empirical version of $l$ on the grid $\{X_i \ ; \ (X_i,Y_i)\in \Ll\}$ in the M2 case.
\begin{definition} \label{def:variance}
The $\L^2$ distance between two classifiers $f$ and $g$ is defined by $d(f,g)$, and is computed as follows:\\
\ni $(i)$ if $\ef$ is constructed via M1, $d^2(f,g):= \E\left[(f(X)-g(X))^2\right]$.\\
\ni $(ii)$ if $\ef$ is constructed via M2, $$d^2(f,g) := d^2_{n_l}(f,g) =  \frac{1}{n_l}\sum_{\{X_i \ ; \ (X_i,Y_i)\in \Ll\}}\left(f(X_i)-g(X_i)\right)^2,$$ 
\end{definition}
\ni As for $\lambda$, $d$ is the empirical version of the $\L^2$ distance on the grid $\{X_i \ ; \ (X_i,Y_i)\in \Ll\}$ with M2.

\Rem If the design is fixed, $\lambda$ and $d$ are different according to the method only through the grid on which they are computed (the grid of method M1 being obtained from the one of method M2 by taking one point out of two). In this case, $\lambda$ and $d$ are no more random.\\

\ni We based our computation of risk bounds for the ERM classifiers provided by CART on recent results (see for instance \cite{MamTsy99}, \cite{Tsy04}, \cite{TsyGee05}, \cite{MasNed06}, \cite{Kol06, Kol06rej}, \cite{Mas07}, \cite{Lec07}, \cite{KohKrz07}). They stem from Vapnik's results (see \cite{Vap98}, \cite{Lug02} for example), showing that, without any assumption on the joint distribution $\P$, the penalty term used in the penalized criterion for the model selection procedure should be taken proportional to $\sqrt{|\Tf|/n_l}$ to obtain classifiers optimal in term of conditional risk (see \cite{Nob02, ScoNow06} for instance). Nevertheless, it has also been shown that, under the overoptimistic zero-error assumption (that is $Y=\eta(X)$ almost surely, where $\eta$ is defined by (\ref{eta})), this penalty term should be taken proportional to $|\Tf|/n_l$, as done in criterion \eqref{cartcrit}. Since we aim at validating the choice of the penalized criterion \eqref{cartcrit} in contexts less restrictive than the zero-error one, we consider weaker assumptions on $\P$.

\subsection{Margin assumptions} \label{subsec:margin}

Margin assumptions are now widely known to improve risk bounds of ERM classifiers in the binary classification context. One of the best-known margin assumptions is that of Mammen and Tsybakov \cite{MamTsy99} that may be written as follows:
\begin{description}
\item {\bf MA(MT)} There exist some constants $C>0$ and $\kappa>1$ such
  that, for all $t>0$,
\begin{eqnarray} \label{marginTsy}
\P\left(|2\eta(X)-1|\leqslant t\right)\leqslant C \ t^{\frac{1}{\kappa-1}},
\end{eqnarray}
\end{description}
where $\eta$ is defined by \eqref{eta}. {\bf MA(MT)} implies the more intuitive assumption considered by Massart and Nedelec in \cite{MasNed06} (see also the slightly weaker condition proposed in \cite{KohKrz07}): taking $t=h\in ]0;1[$ and the limit value $\kappa=1$, {\bf MA(MT)} leads to
\begin{description}
\item {\bf MA(MN)} $\exists h\in ]0;1[ \ \ \ \P\left(|2\eta(X)-1|\leq h\right)=0$.
\end{description}
Assumption {\bf MA(MN)} means that $(X,Y)$ is sufficiently well distributed to ensure that there is no region in $\X$ for which the toss-up strategy could be favored over others: $h$ can be viewed as a measurement of the gap between labels 0 and 1 in the sense that, if $\eta(x)$ is too close to $1/2$, then choosing 0 or 1 will not make a real difference for that $x$. \\ In assumption {\bf MA(MT)}, $\eta$ can be continuous, but has to cross the line $\eta(x)=1/2$ in a non smooth way.\\

\ni From this simple example, the so called {\it margin} $h$ can be viewed as a noise level for the classification problem. From this point of view, margin assumptions have been generalized by Koltchinskii in \cite{Kol06}; they compare directly the loss $l$ defined by \eqref{loss} with some kind of "noise variance" related to the $\L^2$ distance to the Bayes classifier $f^*$: 
\begin{description}
\item {\bf MA(K)} There exists some strictly convex positive function $\varphi$ satisfying $\varphi(0)=0$ such that, $$\forall f:\X \ \rightarrow \ \{0;1\} \ \ \ l(f^*,f) \geqslant \varphi\left(\sqrt{\E\left[(f(X)-f^*(X))^2\right]}\right)$$
\end{description}
It is easy to check that {\bf MA(MT)} and {\bf MA(MN)} imply {\bf MA(K)} with $\varphi(x)=C_{\kappa}\displaystyle{x^{\frac{2\kappa}{2\kappa-1}}}$ and $\varphi(x)=h x^2$ respectively.

\ni \Rem Taking $h>1$ in {\bf MA(MN)} (or more generally $\varphi(x)> x^2$ in {\bf MA(K)}) has no sense since, for any classifier $f$, (see \cite{DevGyoLug96} for instance) $$l(f^*,f)=\E\left[|2\eta(X)-1|\left(f(X)-f^*(X)\right)^2\right]\leqslant \E\left[(f(X)-f^*(X))^2\right].$$

\medskip

\ni {\bf MA(MT)} (with $\kappa>1$) and {\bf MA(MN)} (with $\kappa=1$) lead to risk bounds suggesting that the empirical misclassification rate of $\eft{T}$ have to be penalized by a term proportional to $\left(|\Tf|/n_l\right)^{\kappa/(2\kappa-1)}$ to obtain ERM classifiers optimal in terms of risk (see also \cite{TsyGee05} for instance), while {\bf MA(K)} leads to more general penalty terms given by strictly concave functions of $|\Tf|/n_l$. Hence these margin assumptions make the link between the ``global'' pessimistic case (without any assumption on $\P$) and the zero-error case by considering some noise level of the classification problem. 
More recent results (see \cite{Kol06, Kol06rej}, \cite{ArlBar08} for instance) deal with data-driven penalties based on local Rademacher complexities also derived from margin assumptions. \\

\ni As it can be seen in \cite{Brei84}, the CART pruning algorithm looks at the entire family of pruned subtrees according to criterion \eqref{cartcrit} only if the penalty taken in the criterion is linear. Thus, it follows from the above mentioned results that the following margin assumption has to be fulfilled: 
\begin{description}
\item {\bf MA(1)} $\exists h\in ]0;1[ \ \ \ \forall f:\X \mapsto \{0;1\} \ \ \ \lambda(f^*,f)\geqslant h d^2(f^*,f)$,
\end{description}
where $\lambda$ and $d$ are defined in Definitions \ref{def:perte} and \ref{def:variance} respectively.\\

\ni {\bf Examples}: 
\begin{enumerate}[1) ]
\item Take $X=(X^1,\ldots ,X^d)$ uniformly distributed on $[0;1]^d$. The associated label is designed as follows: if $X^j\leq 1/2$ or $X^j>1/2$ for all $j=1,\ldots ,d$, then $Y=1$ with probability $q$; otherwise $Y=1$ with probability $1-q$.
\item Take $X=(X^1,X^2)$ such that $X^1$ and $X^2$ are independently generated with gaussian distribution ${\mathcal{N}}(0,1)$. The associated label is designed as follows: If $X^1 > 0$ and $X^2>0$ then $Y=1$ with probability $q$, otherwise $Y=1$ with probability $1-q$.
\end{enumerate}
\ni In these two simple examples, if $q\neq 1/2$, {\bf MA(MN)}, and consequently {\bf MA(1)}, is satisfied with any value of $h$ satisfying $0<h<|2q-1|$ in both M1 and M2 cases; indeed $\eta(X)=q$ or $\eta(X)=1-q$, depending on where $X$ falls. Examples in which {\bf MA(1)} fails can be found in \cite{ArlBar08}.\\

\ni Below, we prove that, under {\bf MA(1)}, the penalty used by CART in criterion \eqref{cartcrit} for the pruning step leads to classifiers having good performance.\\ In the remaining part of this paper, the constant $h$ will denote the so called {\it margin}.

\section{Risk Bounds} \label{sec:risk}

This section is devoted to the results obtained on the performance of the CART classifiers for both M1 and M2 methods. These performance are regarded from the risk viewpoint presented in paragraph \ref{subsec:classif}, where classifiers are considered as estimators of the Bayes classifier $f^*$. The risk of the classifier $\ef$ provided by the CART algorithm is compared to those of the collection $\left(\eft{T}\right)_{T\< T_{max}}$ conditionally on the construction of $T_{max}$.\\
We shall first present a general theorem, then give more precise results about the last two parts of the algorithm, which are the pruning algorithm and the final selection by test sample.

\begin{theorem} \label{thfin}
Given $N$ independent pairs of variables $((X_i,Y_i))_{1\leq i\leq N}$ of
common distribution $\P$, with $(X_i,Y_i)\in \X\times \{0,1\}$, let us consider
the estimator $\ef$ (\ref{estim}) of the Bayes classifier $f^*$ (\ref{bayes})
obtained via the CART algorithm as defined in section \ref{sec:cart}. Then we have the following results. 

\smallskip

\ni $(i)$ if $\ef$ is constructed via M1:\\
Suppose that margin assumption {\bf MA(1)} is satisfied. Then, there exist some absolute constants $C$, $C_1$ and $C_2$ such that
\begin{eqnarray}
\E\left[\lambda(f^*,\ef) \  | \ \Ll_1\right] \hspace{-0.2cm} & \leq & \hspace{-0.2cm} C\inf_{T\<
  T_{max}}\left\{\inf_{f\in \F_T}\E\left[\lambda(f^*,f) \ | \
  \Ll_1\right]+\frac{|\Tf|}{hn_2}\right\}+\frac{C_1}{hn_2} \label{part11}\\
                                       &  & +C_2\frac{\log{(n_l)}}{hn_t}.  \label{part12}
\end{eqnarray}

\smallskip

\ni $(ii)$ if $\ef$ is constructed via M2:\\
Let $P_{\Ll}$ be the $\Ll$ sample distribution. Let $V$ be the Vapnik-Chervonenkis dimension of the set of splits used to construct $T_{max}$ and suppose that $n_l\geqslant V$. Let $K$ be the number of pruned subtrees of the sequence provided by the pruning algorithm, and suppose that margin assumption {\bf MA(1)} is satisfied. Then, there exist some absolute constants $C'$, $C'_1$, $C''_1$ and $C_2$ such that, for every $\delta \in ]0;1[$, on a set $\Omega_{\delta}$ verifying $\P_{\Ll}(\Omega_{\delta})\geq 1-\delta$,
\begin{eqnarray}
\E\left[\lambda(f^*,\ef) \  | \ \Ll\right] \hspace{-0.2cm} & \leq & \hspace{-0.2cm} C'\inf_{T\<
  T_{max}}\left\{\inf_{f\in \F_T}\lambda(f^*,f)+\log{\left(\frac{n_l}{V}\right)}
  \frac{|\Tf|}{hn_l}\right\}+\frac{C_{\delta}}{hn_l} \label{part21}\\
                            & & +C_2 \frac{\log{K}}{hn_t}, \label{part22}
\end{eqnarray}
with $C_{\delta}=C'_1+C''_1\log{(1/\delta)}$.
\end{theorem}

\ni Note that the constants appearing in the upper bounds for the risks are
not sharp. We do not investigate the sharpness of the constants here.\\

\ni Several comments can be made on the basis of the results from Theorem \ref{thfin}:
\paragraph{Methods} Both methods M1 and M2 are considered for the following reasons:
\begin{itemize}
\item Since all the risks are considered conditionally on the growing
procedure, the M1 method permits to make a deterministic penalized model
selection and then to obtain sharper upper bounds than the M2 method.
\item On the other hand, the M2 method permits to keep the whole information given by $\Ll$. Indeed, in that case, the sequence of pruned subtrees is not obtained via some plug-in method using a first split of the sample to provide the collection of tree-structured models. This method is the one proposed by Breiman {\it et al.} and it is more commonly applied in practice than the former. We focus
on this method to ensure that it provides classifiers that have good performance in terms of risk.
\end{itemize}
\paragraph{Interpretation of the bounds} For both M1 and M2 methods, the inequality of Theorem \ref{thfin} may be divided into two parts:
\begin{itemize}
\item (\ref{part11}) and (\ref{part21}) correspond to the pruning algorithm. They show that, up to some absolute constant and the final selection, the conditional risk of the final classifier is approximately of the same order as the infimum of the penalized risks of the collection of subtrees of $T_{max}$. The term inside the infimum is of the same form as the penalized criterion (\ref{cartcrit}) used in the pruning algorithm. This shows that, for a sufficiently large temperature $\alpha$, this criterion allows to select convenient subtrees in term of conditional risk.\\ Let us emphasize that the remainder term driving the choice of the penalty is directly proportional to the number of leaves in the M1 method, whereas a multiplicative logarithmic term appears in the M2 method. This term is due to the randomness of the models considered, since the samples used to construct and prune $T_{max}$ are no longer independent.
\item (\ref{part12}) and (\ref{part22}) correspond to the final selection of $\ef$ among the collection of pruned subtrees using $\T$. As $K\leq n_l$, this selection adds a term proportional to $\log{n_l}/n_t$ for both methods, showing that not much is lost when a test sample is used provided that $n_t$ is sufficiently large with respect to $\log{n_l}$. Nevertheless, since we have no idea of the size of the constant $C_2$, it is difficult to deduce a general way of choosing $\T$ from this upper bound. 
\end{itemize}

\paragraph{Consistency results} Since growing and pruning are independent when applying M1, the VC-dimension $V$ of the set of splits ${\mathcal{S}}$ only appears with M2. Thus, in this case, the term $\log{(n_l/V)}$ in the infimum has to be taken into account if $V$ is negligeable in front of $n_l$. Nevertheless, if CART provides models such that
\begin{enumerate}[\hspace{0.5cm} -]
\item the maximal dimension of the models is
$D_N={\mathrm{o}}\left(N/\log N\right)$,
\item the approximation properties of the models are convenient enough
to ensure that the bias tends to zero with increasing sample size $N$,
\end{enumerate}
then we have a result of consistency for $\ef$ provided that $n_t$ is conveniently chosen with respect to $\log{n_l}$.

\paragraph{Role of the margin} It has been shown in \cite{MasNed06} and in \cite{MamTsy99} that, under margin assumptions {\bf MA(MN)} and {\bf MA(MT)} respectively, the ERM estimator of $f^*$ on one model is minimax if $f^*$ belongs to some H\"older classes. This means that, under margin assumption {\bf MA(1)}, the upper bound obtained in Theorem \ref{thfin} for the CART classifier can not be improved. On the other hand, if margin assumption {\bf MA(MT)} is fulfilled, similar bounds are obtained with a remainder term in the infimum proportional to $\left(|\Tf|/n_l\right)^{\kappa/(2\kappa-1)}$. Since $\kappa>1$, this term is subbaditive with respect to $|\Tf|$ (see \cite{Sco05} for full description of subbaditive penalties), so results of \cite{Sco05} can be applied: the subtrees pruned by minimizing a penalized criterion with a penalty proportional to $\left(|\Tf|/n_l\right)^{\kappa/(2\kappa-1)}$ are subtrees of the CART sequence $(T_k)_{1\leq k\leq K}$. So, if $\kappa$ is known, the best solution is to prune $T_{max}$ with the usual pruning algorithm, and then to extract from the sequence obtained in that way the subsequence minimizing the criterion penalized by the subadditive penalty. 

\paragraph{Margin dependent penalties} It is important to point out that the penalty term suggested by the risk bounds depends on margin parameters, which are usually unknown in practice. To withdraw the margin parameter $h$ under margin assumption {\bf MA(1)}, one prunes $T_{max}$ with the pruning algorithm given in Table \ref{tab:pruning}, and then one uses a test sample or cross-validation to select a subtree. If no margin assumption is fulfilled, the procedure of Scott \cite{Sco05} can be applied, with a penalty term proportional to $\sqrt{|\Tf|/n_l}$. Otherwise, the margin parameters have to be estimated.

\paragraph{Optimality of the bounds} Theorem \ref{thfin} also shows that the higher the margin, the smaller the risk, which is intuitive since the inverse of the margin plays the role of the classification noise. Actually, to reach optimality in terms of conditional risk, the penalty should be taken as $ cst \times \left(h^{-1} |\Tf|/n_l\wedge \sqrt{|\Tf|/n_l}\right)$ since, in any case, the remainder term inside the infimum is, at worst, proportional to $\sqrt{|\Tf|/n_l}$. Hence CART will underpenalize trees for which $h \leq \sqrt{|\Tf|/n_l}$, leading to classifiers having an excessive number of leaves. Nevertheless, the condition $h>\sqrt{|\Tf_{max}|/n_l}$ can be controlled during the growing algorithm by forcing the maximal tree's construction to stop earlier, for example. This is obviously difficult to do in practice since it heavily depends on the data and on the size of the learning sample, and is worth being investigated more thoroughly.

\bigskip

\ni The two following subsections give more precise results on the pruning
algorithm for both the M1 and M2 methods, and particularly on the constants
appearing in the penalty function. Subsection
\ref{subsec:dis} validates the discrete selection by test-sample. 

\subsection{Validation of the Pruning algorithm} \label{subsec:prun}

In this section, we focus more particularly on the pruning algorithm
and give trajectorial risk bounds for the classifier associated with  $T_{\a}$, the smallest minimizing subtree for the temperature $\a$ defined in subsection \ref{subsec:algo}. We show that, for a convenient constant $\a$, $\eft{T_{\a}}$ is not far from $f^*$ in terms of its conditional risk. Let us emphasize that the subsample $\T$ plays no role in the two following results.

\subsubsection{$\ef$ constructed via M1} \label{subsec:m1}

Here we assume that $\Ll=\Ll_1\cup \Ll_2$. Thus $T_{max}$ is constructed on the first
set of observations $\Ll_1$ and then pruned with the second set $\Ll_2$
independent of $\Ll_1$. Since the set of pruned subtrees is deterministic according to $\Ll_2$, the selection is made among a deterministic collection of models.\\
\ni For any subtree $T$ of $T_{max}$, let $\F_T$ be the model defined on the leaves of $T$ given by (\ref{eq:model}). Let $\P_{n_2}$ be the empirical misclassification rate on $\Ll_2$ as defined by (\ref{contrast}). Then let us consider the following:
\begin{itemize}
\item For $T\<T_{max}$, $\eft{T}=\argmin_{f\in \F_T}\left[\P_{n_2}(f)\right]$, 
\item For $\a>0$, $T_{\a}$ is the smallest minimizing subtree for the
  temperature $\a$ as defined in subsection \ref{subsec:algo} and $\eft{T_{\a}}=\argmin_{f\in \F_{T_{\a}}}\left[\P_{n_2}(f)\right]$.
\end{itemize}

\begin{proposition} \label{pruM1}
Let $P_{\Ll_2}$ be the product distribution on $\Ll_2$ and let $h$ be the margin given by {\bf MA(1)}. Let $\xi>0$. \\
There exists a large enough positive constant $\a_0>2+\log{2}$ such that, if $\a> \a_0$, then, there exist some nonnegative constants $\Sigma_{\a}$ and $C$ such that 
\begin{eqnarray*}
l(f^*,\eft{T_{\a}}) & \leq  & C_1(\a) \ \underset{T\<
  T_{max}}{\inf}\left\{\inf_{f\in \F_T}l(f^*,f)+h^{-1}\frac{|\Tf|}{n_2}\right\}+C\ h^{-1}\frac{1+\xi}{n_2}
\end{eqnarray*}
on a set $\Omega_{\xi}$ such that $P_{\Ll_2}(\Omega_{\xi})\geqslant 1-\Sigma_{\a}
e^{-\xi}$, where $l$ is defined by (\ref{loss}), $C_1(\a)>\a_0$ and $\Sigma_{\a}$ are increasing with $\a$.
\end{proposition}

\ni We obtain a trajectorial non-asymptotic risk bound on a large probability set, leading to the conclusions given for Theorem \ref{thfin}. Nevertheless, taking an excessive temperature $\a$ will overpenalize and select a classifier having high risk $\E[l(f^*,\eft{T_{\a}}) \ | \ \Ll_1]$. Furthermore, the fact that $C_1(\a)$ and $\Sigma_{\a}$ are increasing with $\a$ suggests that both sides of the inequality grow with $\a$. The choice of the convenient temperature is then critical to make a good compromise between the size of $\E[l(f^*,\eft{T_{\a}}) \ | \ \Ll_1]$ and a large enough penalty term. 

\subsubsection{$\ef$ constructed via M2} \label{subsec:m2}

Here we define the different empirical risks, expected loss and
estimators exactly in the same way as in subsection \ref{subsec:m1}, although $l$ is replaced by the empirical expected loss $\lambda$ on $X_1^{n_l}=\{X_i \ ; \ (X_i,Y_i)\in \Ll\}$ defined in Definition \ref{def:perte}. In this case, we obtain nearly
the same performance for $\eft{T_{\a}}$ despite the fact that the constant
appearing in the penalty term can now depend on $n_l$: 

\begin{proposition} \label{pruM2}
Let $P_{\Ll}$ be the product distribution on $\Ll$, $\lambda$ be the empirical expected loss computed on $\{X_i \ ; \ (X_i,Y_i)\in \Ll\}$, and let $h$ be the margin given by {\bf MA(1)}. Let $\xi>0$ and $$\a_{n_l,V}=2+V/2\left(1+\log{\frac{n_l}{V}}\right).$$
There exists a large enough positive constant $\a_0$ such that, if $\a> \a_0$,
then, there exist some nonnegative constants $\Sigma_{\a}$ and $C'$ such
that 
\begin{eqnarray*}
 \lambda(f^*,\eft{T_{\a}}) & \leq  & C'_1(\a) \ \underset{T\<
 T_{max}}{\inf}\left\{\inf_{f\in \F_T}\lambda(f^*,f)+h^{-1}\a_{n_l,V}\frac{|\Tf|}{n_l}\right\}+C'\ h^{-1}\frac{1+\xi}{n_l}
\end{eqnarray*}
on a set $\Omega_{\xi}$ such that $P_{\Ll}(\Omega_{\xi})\geqslant 1-2\Sigma_{\a}e^{-\xi}$, where $C'_1(\a)>\a_0$ and $\Sigma_{\a}$ are increasing with $\a$.
\end{proposition}

\ni We obtain a similar trajectorial non-asymptotic risk bound on a large probability set. The same conclusions as those derived from M1 hold in this case. Let us just mention that the remainder term $h^{-1}\a_{n_l,V}|\Tf|/n_l$ in the risk bound takes into account the complexity of the collection of
  trees having $|\Tf|$ leaves which can be constructed on $\{X_i \ ; \
  (X_i,Y_i)\in \Ll\}$. Since this complexity is controlled via the
  VC-dimension $V$, $V$ necessarily appears in the penalty term. It differs
  from Proposition \ref{pruM1} in the sense that the models we consider are
  random, so this complexity has to be taken into account to obtain a uniform
  bound. \\
  
\ni {\bf Example}: Let us consider the case where $\mathcal{S}$ is
the set of all half-spaces of $\X={\mathbb{R}}^d$ with axis-parallel frontiers. In this case, if $d\geqslant 3$, $$ \frac{\log{(d)}}{\log 2}-1.18\leq V \leqslant d,$$ consequently, if $n_l\geqslant d$, we obtain a penalty proportional to
$$\left(\frac{4+d\left(1+\log{[n_l\log 2/(\log d-2\log 2)]}\right)}{2h}\right)\frac{|\Tf|}{n_l}.$$ So, if CART provides some minimax estimator on a class of functions, the $\log{n_l}$ term always appears for $f^*$ in this class when working in a linear space of low dimension. 

\subsection{Final Selection} \label{subsec:dis}

We focus here on the selection of the classifier $\ef$ among the collection $(\eft{T_k})_{1\leq k\leq K}$ provided by the pruning algorithm as defined in subsection \ref{subsec:algo}. Let us recall that $\ef$ is defined by
$$\ef=\underset{\{\eft{T_k} ; 1\leq k\leq
  K\}}{\argmin}\left[\P_{n_t}(\eft{T_k})\right],$$ 
\ni where $\P_{n_t}$ is the empirical misclassification rate on $\T$ defined by \eqref{contrast}.\\
\ni The performance of this classifier can be compared to the performance of
  the collection $(\hat{f}_{T_k})_{1\leq k\leq K}$ by the following:

\begin{proposition} \label{FS}
~\\
Let $\lambda$ be the loss defined in Definition \ref{def:perte}. For both methods M1 and M2, there exist three absolute constants $C''>1$, $C_1'>3/2$ and $C_2'>3/2$ such that
\begin{eqnarray*}
\E\left[\lambda(f^*,\ef) \ | \ \Ll\right] & \leq & C'' \ \underset{1\leq
  k\leq K}{\inf}\lambda(f^*,\eft{T_k})+C_1' \ h^{-1}\frac{\log{K}}{n_t}+h^{-1}\frac{C_2'}{n_t}, 
\end{eqnarray*}
where $K$ is the number of pruned subtrees extracted during the pruning algorithm.
\end{proposition}

\section{Concluding Remarks} \label{sec:ccl}

We have proven that CART provides convenient classifiers in terms of conditional risk under the margin assumption {\bf MA(1)}. As for the regression case, the properties of the growing algorithm need to be analyzed to obtain full unconditional upper bounds. Results on the performance of theoretical procedures in which CART is viewed as a forward algorithm to approximate an ideal, but intractable, binary tree are given in \cite{GeyMar11}. Although they do not validate any concrete algorithm as done here, these results confirm that the penalty term used in penalized criterion \eqref{cartcrit} is well chosen under {\bf MA(1)}. \\ 

\ni The remarks made after Theorem \ref{thfin} on the size of the margin $h$ enlarge our perspectives for the application of CART in practice. Among such perspective, we may
\begin{itemize}
\item use the slope heuristic (see for example \cite{ArlMas09}) to select a classifier among a collection,
\item search for a robust manner to determine if the margin assumption is fulfilled, allowing to use the blind selection by test sample.
\end{itemize}
Some track to estimate the margin $h$ if assumption {\bf MA(1)} is fulfilled could be to use mixing procedures as boosting (see \cite{Brei98} \cite{FreSha97} for example). Hence, this estimate could be used in the penalized criterion to help find the convenient temperature. It could also give an idea of the difficulty to classify the considered data and henceforth to help choose the most adapted classification method.

\section*{Acknowledgements}

I would like to thank an anonymous referee for numerous remarks and suggestions which helped to improve the presentation of this paper.

\section{Proofs} \label{sec:append}

Let us start with a preliminary result.

\subsection{Local Bound for Tree-Structured Classifiers} \label{subsec:bound}

Let $(X,Y)\in \X \times \{0;1\}$ be a pair of random variables and
$\{(X_1,Y_1),\ldots,(X_n,Y_n)\}$ be $n$ independent copies of $(X,Y)$. Then given two classifiers $f$
and $g$, let us define $$d_n^2(f,g)=\frac{1}{n}\sum_{i=1}^n\left(f(X_i)-g(X_i)\right)^2.$$
Let $\Mn^*$ be the set of all possible tree-structured partitions that can be constructed on the grid $X_1^n$, corresponding to trees having all possible splits in ${\mathcal{S}}$ and all possible forms without taking account of the response variable $Y$. So $\Mn^*$ only depends on the grid $X_1^n$ and is independent of the variables $(Y_1,\ldots,Y_n)$. Hence, for a tree $T\in \Mn^*$, define $$\F_T=\left\{\sum_{t\in \Tf}a_t\1_t \ ; \ (a_t)\in \{0,1\}^{|\Tf|}\right\},$$ where
$\Tf$ refers the set of the leaves of $T$. Then, for any $f\in \F_T$ and any
$\s>0$, define
\begin{eqnarray*}
B_T(f,\s) & = & \left\{g\in \F_T \ ; \ d_n(f,g)\leq \s\right\}
\end{eqnarray*}
For each classifier $f : \X \rightarrow \{0,1\}$, let us define the
empirical contrast of $f$ recentered conditionally on $X_1^n$
\begin{equation} \label{gammabar}
\overline{\P}_n(f)=\P_n(f)-\E[\P_n(f) \ | \ X_1^n],
\end{equation}
where $\P_n$ is defined for any given classifier $f$ by
\begin{eqnarray*} 
\P_n(f) & = & \frac{1}{n}\sum_{i=1}^n\1_{f(X_i)\neq Y_i}.
\end{eqnarray*}
\Rem If $\P_n$ is evaluated on a sample $(X_i')$ independent of $X_1^n$, it is
easy to check that the bounds we obtain in what follows are still valid by
taking the population distance $$d^2(f,g)=\E\left[(f(X)-g(X))^2\right]$$ instead of its empirical version $d_n$.

\medskip

\ni We have the following result:
\begin{lemma} \label{locbd}
For any $f\in \F_T$ and any $\s>0$ 
$$\E\left[\sup\limits_{g\in B_{T}(f,\sigma)}\ |\overline{\P}_n(g)-\overline{\P}_n(f)| \ | \ X_1^n\right]\leq 2\ \sigma \sqrt{\frac{|\Tf|}{n}}.$$
\end{lemma}

\begin{proof}
First of all, let us mention that, since the different variables we consider
take values in $\{0;1\}$, we have for all $x\in \X$ and all $y\in \{0,1\}$
$$\1_{g(x)\neq y}-\1_{f(x)\neq y}=(g(x)-f(x))(1-2\1_{y=1}),$$ yielding
\begin{eqnarray*}
\overline{\P}_n(g)-\overline{\P}_n(f) & = & \frac{1}{n}\sum_{i=1}^n\left(g(X_i)-f(X_i)\right)(1-2\1_{Y_i=1})\\
                                                             &     & -\E\left[\frac{1}{n}\sum_{i=1}^n\left(g(X_i)-f(X_i)\right)(1-2\1_{Y_i=1}) \ | \ X_1^n\right].  
\end{eqnarray*}
\ni Let us now consider a Rademacher sequence of random signs $(\v_i)_{1\leq i\leq
n}$ independent of $(X_i,Y_i)_{1\leq i\leq n}$. Then, one has by a symmetrization argument $$\E\left[\sup_{g\in
  B_T(f,\sigma)}|\overline{\P}_n(g)-\overline{\P}_n(f)| \ | \ X_1^n\right]\leq \E\left[\sup_{g\in
  B_T(f,\sigma)}\frac{2}{n}\left|\sum_{i=1}^n\v_i(g(X_i)-f(X_i))(1-2\1_{Y_i=1})\right| \ | \
X_1^n\right].$$ Since $g$ and $f$ belong to $\F_T$, we have that
$$g-f=\sum_{t\in \Tf}(a_t-b_t)\varphi_t,$$ where each $(a_t,b_t)$ takes values in
$[0,1]^2$ and $(\varphi_t)_{t\in \Tf}$ is an orthonormal basis of $\F_T$
adapted to $\Tf$ (i.e some normalized characteristic functions). Then, by
applying the Cauchy-Schwarz inequality, since $g\in
  B_T(f,\sigma)$, $d_n^2(f,g)=\sum_{t\in \Tf}(a_t-b_t)^2\leq
  \sigma^2$, we obtain that
\begin{eqnarray*}
\left|\sum_{i=1}^n\v_i(g(X_i)-f(X_i))(1-2\1_{Y_i=1})\right| & \leq & \sqrt{\sum_{t\in
  \Tf}(a_t-b_t)^2}\sqrt{\sum_{t\in
  \Tf}\left(\sum_{i=1}^n\v_i(1-2\1_{Y_i=1})\varphi_t(X_i)\right)^2}\\
  & \leq & \sigma \sqrt{\sum_{t\in \Tf}\left(\sum_{i=1}^n\v_i(1-2\1_{Y_i=1})\varphi_t(X_i)\right)^2}.
\end{eqnarray*}
Finally, since $(\v_i)_{1\leq i\leq n}$ and $(1-2\1_{Y_i=1})_{1\leq i\leq n}$ take their values in $\{-1;1\}$, $(\v_i)_{1\leq i\leq n}$ are centered and independent of $(X_i,Y_i)_{1\leq i\leq n}$, and since, by definition, for each $t\in \Tf$ $n^{-1}\sum_{i=1}^n\varphi_t^2(X_i)=1$, Jensen's inequality implies
$$\E\left[\sup\limits_{g\in B_{T}(f,\sigma)}\
  |\overline{\P}_n(g)-\overline{\P}_n(f)| \ | \ X_1^n\right]\leq 2\frac{\sigma}{n}
\sqrt{\sum_{t\in \Tf}\sum_{i=1}^n\varphi_t^2(X_i)}\leq 2\sigma
\sqrt{\frac{|\Tf|}{n}}.$$ 
\end{proof}

\subsection{Proof of Proposition \ref{pruM1}} \label{subsec:proofM1}

To prove Proposition \ref{pruM1}, we adapt results from Massart \cite[Theorem 4.2]{Mas00}, and Massart and N\'ed\'elec \cite{MasNed06} (see also Massart {\it et.al.} \cite{Mas07}). 

\medskip

\ni Let $n=n_2$. Let us give a sample $\Ll_2=\{(X_1,Y_1),\ldots,(X_n,Y_n)\}$ of the random variable
$(X,Y)\in \X \times [0,1]$, where $\X$ is a measurable space and let $f^*\in
{\mathcal{F}}\subset \{f : \X \mapsto [0,1] \ ; \ f\in
\L^2(\X)\}$ be the unknown function to be recovered. Assume $(\Fm)_{m\in \Mn}$ is a countable
collection of countable models included in ${\mathcal{F}}$. Let us give a penalty function $\pen:\Mn \longrightarrow
\mathbb{R}_{+}$, and $\g : {\mathcal{F}} \times (\X \times [0,1])
\longrightarrow \mathbb{R}_+$ a contrast function, i.e. $\g$ such that $f \mapsto \E\left[\g(f,(X,Y))\right]$ is
convex and minimum at point $f^*$. Hence define for all $f\in
{\mathcal{F}}$ the expected loss
$l(f^*,f)=\E\left[\g(f,(X,Y))-\g(f^*,(X,Y))\right]$.\\ Finally let
\begin{eqnarray} \label{EC}
\g_n=\frac{1}{n}\sum_{i=1}^n\g(.,(X_i,Y_i))
\end{eqnarray}
be the empirical contrast
associated with $\g$. For example, in the classification context, $\g(f,(x,y))=\1_{f(x)\neq y}$, leading to the classical loss as defined by \eqref{loss}, and the classical empirical misclassification rate $\P_n$ as defined by \eqref{contrast}. Hence, if the collection of models $\Mn$ has finite-dimensional models with dimension $|m|$, the penalty function can be taken as $\pen(m)=cst \ \times \ |m|$ for instance.\\ Then let $\hat{m}$ be defined as $$\hat{m}=\underset{m\in
  \Mn}{\argmin}\left[\g_n(\efm)+\pen(m)\right]$$ where $\efm=\argmin_{g\in \Fm}\g_n(g)$ is the minimum empirical contrast estimator of $f^*$ on $\Fm$. The final estimator of $f^*$ is
\begin{equation} \label{es}
\ef=\hat{f}_{\hat{m}}.
\end{equation}  

\medskip

\ni One makes the following assumptions:\\
${\mathbf{H_1}}$: $\g$ is bounded by 1, which is not a restriction since all the
functions we consider take values in $[0,1]$.\\
${\mathbf{H_2}}$: Assume there exist $c\geqslant (2\sqrt{2})^{-1/2}$ and some (pseudo-)distance $d$ such that, for every pair $(f,g)\in {\mathcal{F}}^2$, one has
$${\mathrm{Var}}\left[\g(g,(X,Y))-\g(f,(X,Y))\right]\leq d^2(g,f),$$ and
particularly for all $f\in {\mathcal{F}}$ $$d^2(f^*,f)\leq c^2l(f^*,f).$$
${\mathbf{H_3}}$: For any positive $\s$ and for any $f\in \Fm$,
let us define $$B_m(f,\s)=\left\{g\in \Fm \ ; \ d(f,g)\leq
  \s\right\}$$ where $d$ is given by assumption ${\mathbf{H_2}}$. Let
$\bar{\g}_n=\g_n(.)-\E[\g_n(.)]$. We now assume that
for any $m\in \Mn$, there exists some continuous function $\phi_m$ mapping
${\mathbb{R}}_+$ onto ${\mathbb{R}}_+$ such that $\phi_m(0)=0$, $\phi_m(x)/x$ is
non-increasing and $$\E\left[\sup_{g\in
    B_m(f,\s)}|\bar{\g}_n(g)-\bar{\g}_n(f)|\right]\leq \phi_m(\s)$$
for every positive $\s$ such that  $\phi_m(\s)\leq \s^2$. Let $\v_m$ be the
unique solution of the equation $\phi_m(cx)=x^2$ , $x>0$.\\

\ni One gets the following result:

\begin{theorem} \label{thind}
Let $\{(X_1,Y_1),\ldots,(X_n,Y_n)\}$ be a sample of independent realizations of
the random pair $(X,Y)\in \X\times [0,1]$. Let $\left(\Fm\right)_{m\in \Mn}$
be a countable collection of models included in some countable family ${\mathcal{F}}\subset \{f : \X \mapsto [0,1] \ ; \ f\in
\L^2(\X)\}$. Consider some penalty function $\pen:\Mn \longrightarrow
\mathbb{R}_{+}$ and the corresponding penalized estimator $\ef$ (\ref{es}) of
the target function $f^*$. Take a family of weights $(x_m)_{m\in \Mn}$ such that 
\begin{equation} \label{sigma}
\Sigma=\sum\limits_{m\in \Mn}e^{-x_m}<+\infty.
\end{equation}
Assume that assumptions ${\mathbf{H_1}}$, ${\mathbf{H_2}}$ and ${\mathbf{H_3}}$
hold.\\
 Let $\xi>0$. Hence, given some absolute constant $C>1$, there exist some
positive constants $K_1$ and $K_2$ such that, if for all $m\in \Mn$
$$\pen(m)\geqslant K_1\v_{m}^2+K_2c^2\frac{x_m}{n},$$ then, with probability
larger than $1-\Sigma e^{-\xi}$, $$l(f^*,\ef)\leq C \ \inf_{m\in
  \Mn}\left[l(f^*,\Fm)+\pen(m)\right]+C'\ c^2 \frac{1+\xi}{n},$$ where
$l(f^*,\Fm)=\inf_{f_m\in \Fm}l(f^*,f_m)$ and the constant $C'$ only depends on
$C$.
\end{theorem}

\begin{proof}
The proof is inspired from Massart \cite{Mas00} and Massart {\it et.al.} \cite{Mas07}. We give only sketches of proofs since those are now routine results in the model selection area (see \cite{Mas07} for a fuller overview).\\

\ni Let $m\in \Mn$ and $f_m\in \Fm$. The definition of the expected loss and the
fact that $$\g_n(\ef)+\pen(\hat{m})\leq \g_n(f_m)+\pen(m)$$ lead to the
following inequality:

\begin{equation} \label{eqprinc}
l(f^*,\ef) \leq  l(f^*,f_m)+\bar{\g}_n(f_m)-\bar{\g}_n(\ef)+\pen(m)-\pen(\hat{m})
\end{equation}
where $\bar{\g}_n$ is defined by (\ref{gammabar}). The general principle is
now to concentrate $\bar{\g}_n(f_m)-\bar{\g}_n(\ef)$ around its expectation
in order to offset the term $\pen(\hat{m})$. Since $\hat{m}\in \Mn$, we proceed by bounding
$\bar{\g}_n(f_m)-\bar{\g}_n(\hat{f}_{m'})$ uniformly in $m'\in \Mn$.  For $m'\in \Mn$ and $f\in {\mathcal{F}}_{m'}$, let us define
$$w_{m'}(f)=\left[\sqrt{l(f^*,f_m)}+\sqrt{l(f^*,f)}\right]^2+y_{m'}^2,$$ with
$y_{m'}\geqslant \v_{m'}$, where $\v_{m'}$ is defined by assumption
${\mathbf{H_3}}$. Hence let us define $$V_{m'}=\sup_{f\in
  {\mathcal{F}}_{m'}}\frac{\bar{\g}_n(f_m)-\bar{\g}_n(f)}{w_{m'}(f)}.$$ Then
(\ref{eqprinc}) becomes 
\begin{eqnarray*}
l(f^*,\ef) & \leq & l(f^*,f_m)+V_{\hat{m}}w_{\hat{m}}(\ef)+\pen(m)-\pen(\hat{m})
\end{eqnarray*}
Since $V_{m'}$ can be written as $$V_{m'}=\sup_{f\in
  {\mathcal{F}}_{m'}}\nu_n\left(\frac{\g(f_m,.)-\g(f,.)}{w_{m'}(f)}\right),$$ where
  $\nu_n$ is the recentered empirical measure, we bound $V_{m'}$ uniformly in
  $m'\in \Mn$ by using Rio's version of Talagrand's inequality, whose first version can be found in \cite{Rio02}, and recalled here: if ${\mathcal{F}}$ is a countable family of measurable functions such that, for some positive constants $v$ and $b$, one has for all $f\in {\mathcal{F}}$ $P(f^2)\leq v$ and $\|f\|_{\infty}\leq b$, then for every positive $y$, the following inequality holds for $Z=\sup_{f\in {\mathcal{F}}}(P_n-P)(f)$ $${\mathbb{P}}\left[Z-\E(Z)\geqslant \sqrt{2\frac{(v+4b\E(Z))y}{n}}+\frac{by}{n}\right]\leqslant e^{-y}.$$ 
To proceed, we need to check the two bounding assumptions. First, since by
 assumption ${\mathbf{H_1}}$ the contrast $\g$ is bounded by $1$, we have that, for each $f\in {\mathcal{F}}_{m'}$, 
\begin{eqnarray} \label{unifbound}
\left|\frac{\g(f,.)-\g(f_m,.)}{w_{m'}(f)}\right|  & \leq & \frac{1}{y_{m'}^2}.
\end{eqnarray}

\ni Second, by using assumption ${\mathbf{H_2}}$, we have that, for each $f\in {\mathcal{F}}_{m'}$,
\begin{eqnarray} \label{varbound}
{\mathrm{Var}}\left[\frac{\g(f,(X,Y))-\g(f_m,(X,Y))}{w_{m'}(f)}\right] & \leq & \frac{c^2}{4y_{m'}^2}.
\end{eqnarray}
Then, by Rio's inequality, we have for every $x>0$ $$P\left[V_{m'}\geqslant
  \E(V_{m'})+\sqrt{\frac{c^2+16\E(V_{m'})}{2ny_{m'}^2}x}+\frac{x}{ny_{m'}^2}\right]\leq
  e^{-x}.$$ Let us take $x=x_{m'}+\xi$, $\xi>0$, where $x_{m'}$ is given by
  (\ref{sigma}). Then, by summing up over $m'\in \Mn$, we obtain that for all
  $m'\in \Mn$
  $$V_{m'}\leq
  \E(V_{m'})+\sqrt{\frac{c^2+16\E(V_{m'})}{2ny_{m'}^2}(x_{m'}+\xi)}+\frac{x_{m'}+\xi}{ny_{m'}^2}$$
  on a set $\Omega_{\xi}$ such that $P(\Omega_{\xi})\geqslant 1-\Sigma e^{-\xi}$. We now need to bound $\E(V_{m'})$ in
  order to obtain an upper bound for $V_{m'}$ on the set of large
  probability $\Omega_{\xi}$. By using techniques similar to Massart {\it et al.}'s \cite{MasNed06}, we obtain the following inequality via the monoticity of $x\mapsto \phi(x)/x$ and the assumption $c\geqslant (2\sqrt{2})^{-1/2}$: for all $m'\in \Mn$, let $u_{m'}\in {\mathcal{F}}_{m'}$ be defined by
   $$l(f^*,u_{m'})\leq 2\inf_{z\in {\mathcal{F}}_{m'}}l(f^*,z).$$
  Then we have $$\E(V_{m'})\leq
   \E\left[\sup_{z\in
  {\mathcal{F}}_{m'}}\frac{|\bar{\g}_n(z)-\bar{\g}_n(u_{m'})|}{w_{m'}(z)}\right]+\E\left[\frac{|\bar{\g}_n(u_{m'})-\bar{\g}_n(f_m)|}{\inf_{z\in
   {\mathcal{F}}_{m'}}[w_{m'}(z)]}\right].$$
 For every $z\in {\mathcal{F}}_{m'}$, let
  $$\omega_{m'}^2(z)=l(f^*,u_{m'})+\E\left[\g(z,(X,Y))-\g(u_{m'},(X,Y))\right]_+.$$
 Then, since
\begin{eqnarray*}
l(f^*,z) & = & \E\left[\g(z,(X,Y))-\g(f^*,(X,Y))\right]\\
l(f^*,z) & = &l(f^*,u_{m'})+\E\left[\g(z,(X,Y))-\g(u_{m'},(X,Y))\right],
\end{eqnarray*}
 Then we have
 \begin{eqnarray} \label{omega}
 l(f^*,z)\leq \omega_{m'}^2(z)\leq 5 \ l(f^*,z).
 \end{eqnarray}

 \ni On the one hand we have $w_{m'}(z)\geqslant l(f^*,z)+y_{m'}^2\geqslant
   (1/5)\omega_{m'}^2(z)+y_{m'}^2$ for every $z\in {\mathcal{F}}_{m'}$. Hence
   $$\E\left[\sup_{z\in
   {\mathcal{F}}_{m'}}\frac{|\bar{\g}_n(z)-\bar{\g}_n(u_{m'})|}{w_{m'}(z)}\right]\leq 5 \
   \E\left[\sup_{z\in
   {\mathcal{F}}_{m'}}\frac{|\bar{\g}_n(z)-\bar{\g}_n(u_{m'})|}{\omega_{m'}^2(z)+5y_{m'}^2}\right].$$
   Furthermore we have $$\E\left[\sup_{\{z \ ;\ \omega_{m'}(z)\leq
   \v\}}|\g_n(z)-\g_n(u_{m'})|\right]\leq \E\left[\sup_{\{z \ ; \ l(f^*,z)\leq
   \v^2\}}|\g_n(z)-\g_n(u_{m'})|\right],$$ and, if $l(f^*,z)\leq \v^2$, then
   $l(f^*,u_{m'})\leq 2\v^2$ and $d(z,u_{m'})\leq d(f^*,z)+d(s*,u_{m'})\leq
   c\v+c\v\sqrt{2}$. Hence we get that
   $d(z,u_{m'})\leq (1+\sqrt{2})c\v\leq 2c\v\sqrt{2}$. \\ Let us now suppose
   that $\v\geqslant \v_{m'}$. Then we have by monoticity of $x\mapsto
   \phi(x)/x$ and by definition of $\v_{m'}$ that 
   $$\frac{\phi_{m'}(2c\v\sqrt{2})}{(2c\v\sqrt{2})^2}\leq
   \frac{\phi_{m'}(c\v)}{c^2\v^22\sqrt{2}}\leq
   \frac{\phi_{m'}(c\v_{m'})}{c^2\v_{m'}^22\sqrt{2}}\leq 1$$ since $c\geqslant (2\sqrt{2})^{-1/2}$.\\ So,
   by assumption ${\mathbf{H_3}}$, we finally obtain that, for all $\v\geqslant \v_{m'}$,
   $$\E\left[\sup_{\{z \ ; \ \omega_{m'}(z)\leq \v\}}|\g_n(z)-\g_n(u_{m'})|\right]\leq
   \E\left[\sup_{\{z \ ; \ d(z,u_{m'})\leq
   2c\v\sqrt{2}\}}|\g_n(z)-\g_n(u_{m'})|\right]\leq \phi_{m'}(2c\v\sqrt{2}).$$
 So we can apply Lemma 5.5 in \cite{MasNed06} and use the monoticity of $x\mapsto \phi_{m'}(x)/x$ to obtain that
   $$\E\left[\sup_{z\in
   {\mathcal{F}}_{m'}}\frac{|\bar{\g}_n(z)-\bar{\g}_n(u_{m'})|}{w_{m'}(z)}\right] \leq
   4\frac{\phi_{m'}(2c\sqrt{10}y_{m'})}{y_{m'}^2}\leq 8\sqrt{10} \ \frac{\phi_{m'}(cy_{m'})}{y_{m'}^2}.$$ Hence, since $y_{m'}\geqslant
   \v_{m'}$ and $x\mapsto \phi_{m'}(cx)/x$ is nonincreasing, we get by
   definition of $\v_{m'}$ $$\E\left[\sup_{z\in
   {\mathcal{F}}_{m'}}\frac{|\bar{\g}_n(z)-\bar{\g}_n(u_{m'})|}{w_{m'}(z)}\right] \leq
   8\sqrt{10}\frac{\phi_{m'}(c\v_{m'})}{y_{m'}\v_{m'}}\leq 8\sqrt{10}\frac{\v_{m'}}{y_{m'}}.$$
   On the other hand, let us notice that
 \begin{eqnarray*} 
 \inf_{z\in {\mathcal{F}}_{m'}}w_{m'}(z) & \geqslant & 2y_{m'}\inf_{z\in
   {\mathcal{F}}S_{m'}}[\sqrt{l(f^*,z)}+\sqrt{l(f^*,f_m)}]\\
                            & \geqslant & \frac{y_{m'}\sqrt{2}}{c}d(u_{m'},f_m),
 \end{eqnarray*}
 hence $$\E\left[\frac{|\bar{\g}_n(u_{m'})-\bar{\g}_n(f_m)|}{\inf_{z\in
   {\mathcal{F}}_{m'}}[w_{m'}(z)]}\right]\leq
   c(y_{m'}\sqrt{2})^{-1}\E\left[\frac{|\bar{\g}_n(u_{m'})-\bar{\g}_n(f_m)|}{d(u_{m'},f_m)}\right],$$
   leading by Jensen's inequality to
   $$\E\left[\frac{|\bar{\g}_n(u_{m'})-\bar{\g}_n(f_m)|}{\inf_{z\in
   {\mathcal{F}}_{m'}}[w_{m'}(z)]}\right]\leq
   c(y_{m'}\sqrt{2})^{-1}\frac{\sqrt{{\mathrm{Var}}\left[\bar{\g}_n(u_{m'})-\bar{\g}_n(f_m)\right]}}{d(u_{m'},f_m)}\leq
   \frac{c}{y_{m'}\sqrt{2n}}.$$ Then we get for all $m'\in \Mn$

  $$\E[V_{m'}]\leq
  \frac{8\sqrt{10}\v_{m'}+c(2n)^{-1/2}}{y_{m'}}.$$ Hence, taking
  $$y_{m'}=K\left[8\sqrt{10}\v_{m'}+c(2n)^{-1/2}+c\sqrt{\frac{x_{m'}+\xi}{n}}\right]$$
  with $K>0$, we obtain that, on $\Omega_{\xi}$, for all
  $m'\in \Mn$, $$V_{m'}\leq
  \frac{1}{K}\left[1+\sqrt{\frac{1}{2}\left(1+\frac{8}{K\sqrt{2}}\right)}+\frac{1}{2K\sqrt{2}}\right].$$
 So we finally obtain that, on the set $\Omega_{\xi}$, 
 \begin{equation} \label{eqfin}
 l(f^*,\ef)\leq
 l(f^*,f_m)+K'w_{\hat{m}}(\ef)+\pen(m)-\pen(\hat{m}),
 \end{equation}
with  $$K'=\frac{1}{K}\left[1+\sqrt{\frac{1}{2}\left(1+\frac{8}{K\sqrt{2}}\right)}+\frac{1}{2K\sqrt{2}}\right].$$
\ni Finally, by using repeatedly the elementary inequality $(\alpha+\beta)^2\leq 2\alpha^2+2\beta^2$ to bound $y_{\hat{m}}^2$ and $w_{\hat{m}}(\ef)$, we derive that, on the one hand, $$y_{\hat{m}}^2\leq 4K^2\left[640\v_{\hat{m}}^2+\frac{c^2}{2n}+c^2 \ \frac{x_{\hat{m}}+\xi}{2n}\right],$$ and, on the other hand, $$w_{\hat{m}}(\ef)\leq 2l(f^*,\ef)+2l(f^*,f_m)+y_{\hat{m}}^2.$$ Hence
the following inequality holds on $\Omega_{\xi}$ for any $m\in \Mn$ and any
$f_m\in \Fm$:
\begin{eqnarray*}
\left(1-2K'\ \right)l(f^*,\ef) & \leq &
\left(1+2K'\right)l(f^*,f_m)+\pen(m)+ 2K'K^2\frac{\xi}{n}+\frac{2c^2K'K^2}{n}\\
 & & +5\times2^{9}K'K^2\v_{\hat{m}}^2+ 2c^2K'K^2\frac{x_{\hat{m}}}{n}-\pen(\hat{m}),
\end{eqnarray*}
with
$$K'=\frac{C-1}{2(C+1)}, \ \ \ \
K_1=5\times2^{9}K'K^2, \ \ \ \ K_2= 2K'K^2.$$ 
\end{proof}

{\underline{\it Application to classification trees}}:

\medskip

Let us now suppose that $(X,Y)$ takes values in $\X\times \{0,1\}$. The
contrast is taken as $\g(f,(X,Y))=\1_{f(X)\neq Y}$, the expected loss is
defined by (\ref{loss}), and the collection of models is
$(\F_T)_{T\< T_{max}}$. The models and the collection are countable
since there is a finite number of functions in each $\F_T$, and a finite number
of nodes in $T_{max}$. Since we are working conditionally on $\Ll_1$, we can
apply Theorem \ref{thind} directly with $\Ll_2$. To check assumption
${\mathbf{H_2}}$, let us first note that, since all the variables we consider
take values in $\{0,1\}$, we have the following for all classifiers $f$ and $g$
\begin{eqnarray} \label{egalite}
\left(\g(f,(X,Y))-\g(g,(X,Y))\right)^2 & = & \left(\1_{Y\neq f(X)}-\1_{Y\neq g(X)}\right)^2\\
                                       & = & (f(X)-g(X))^2.
\end{eqnarray}
Then, if we take $d^2(f,g)=\E\left[(f(X)-g(X))^2\right]$,
we have that, for all classifiers $f$ and $g$,
${\mathrm{Var}}\left[\g(g,(X,Y))-\g(f,(X,Y))\right]\leq d^2(f,g)$. Moreover,
with the margin condition {\bf MA(1)}, we have that
\begin{eqnarray}\label{module}
l(f^*,f) & \geqslant & h d^2(f^*,f),
\end{eqnarray}
hence assumption ${\mathbf{H_2}}$ is checked with $d$ and
$c^2=1/h$, where $h$ is the margin. By definition of $h$, we have $h\leq 1\leq
2\sqrt{2}$, and then $c\geqslant (2\sqrt{2})^{-1/2}$.\\ Then, assumption
${\mathbf{H_3}}$ is checked by Lemma \ref{locbd}
with $\phi_T(x)=2x\sqrt{|\Tf|/n}$. Hence, Theorem \ref{thind} is verified with $\v_T=\sqrt{1/h}\sqrt{|\Tf|/n}$.\\ 
Finally, to choose a convenient family of weights $(x_{_T})_{T\< T_{max}}$, taking $x_{_T}=\theta |\Tf|$, with $\theta>2\log2$ independent of $|\Tf|$ as done in \cite{GeyNed05}, we
immediately obtain $\Sigma_{\a}=\Sigma_{\theta}< +\infty$. Then, we get
proposition \ref{pruM1} by Theorem \ref{thind}.

\subsection{Proof of Proposition \ref{pruM2}} \label{subsec:proofM2}

Let $n=n_l$ and let $X_1^n$ denote the sample $\{X_i \ ; \ (X_i,Y_i)\in \Ll\}$.

\medskip
 
\ni First we generalize Theorem \ref{thind} to random models, and then we
apply it to CART. Let $(X,Y)$, ${\mathcal{F}}$, $f^*\in {\mathcal{F}}$, $\Ll=\{(X_1,Y_1),\ldots,(X_n,Y_n)\}$, $\g$ and
$\g_n$ be defined as in subsection \ref{subsec:proofM1}. Finally, let us rewrite the
expected loss of $f\in {\mathcal{F}}$ conditionally on $X_1^n$ as in Definition \ref{def:perte}, that is
$$\lambda(f^*,f)=\E\left[\P_{n_l}(f)-\P_{n_l}(f^*) \ |  \ X_1^n \right].$$ Let us consider a collection of at most countable models $(\Fm)_{m\in \Mn^*}$ and a
subcollection $(\Fm)_{m\in \Mn}$, where $\Mn\subset \Mn^*$ may depend on
$\{(X_1,Y_1),\ldots,(X_n,Y_n)\}$. Finally, let us consider a penalty function
$\pen : \Mn \mapsto \R_+$ and let us define the estimator $\ef$ of $f^*$ as
follows: let $$\hat{m}=\argmin_{m\in \Mn}[\g_n(\hat{f}_m)+\pen(m)],$$ where
$\hat{f}_m=\argmin_{f\in \Fm}\g_n(f)$ is the minimum contrast estimator of
$f^*$ on $\Fm$. Then $\ef=\hat{f}_{\hat{m}}$.

\medskip

\ni Let us make the following assumptions.\\
${\mathbf{H_1}}$: $\g$ is bounded by 1.\\
${\mathbf{H_2}}$: Assume there exist $c\geqslant (2\sqrt{2})^{-1/2}$ and some
(pseudo-)distance $d_n$ (that may depend on $X_1^n$) such that, for every pair $(g,f)\in {\mathcal{F}}^2$, one has
$${\mathrm{Var}}\left[\g(g,(X,Y))-\g(f,(X,Y)) \ | \ X_1^n\right]\leq d_n^2(g,f),$$ and
particularly for all $f\in {\mathcal{F}}$ $$d_n^2(f^*,f)\leq c^2\lambda(f^*,f).$$
${\mathbf{H_3}}$: For any positive $\s$ and for any $f\in \Fm$,
let us define $$B_m(f,\s)=\left\{g\in \Fm \ ; \ d_n(f,g)\leq
  \s\right\}$$ where $d_n$ is given by assumption ${\mathbf{H_2}}$. Let
$\bar{\g}_n$ be defined as (\ref{gammabar}). We now assume that
for any $m\in \Mn$, there exists some continuous function $\phi_m$ mapping
${\mathbb{R}}_+$ onto ${\mathbb{R}}_+$ such that $\phi_m(0)=0$, $\phi_m(x)/x$ is
non-increasing and $$\E\left[\sup_{g\in
    B_m(f,\s)}|\bar{\g}_n(g)-\bar{\g}_n(f)| \ | \ X_1^n\right]\leq \phi_m(\s)$$
for every positive $\s$ such that  $\phi_m(\s)\leq \s^2$. Let $\v_m$ be the
unique solution of the equation $\phi_m(cx)=x^2$ , $x>0$.\\

\ni One gets the following result.

\begin{theorem} \label{thss}
Let $\Ll=\{(X_1,Y_1),\ldots,(X_n,Y_n)\}$ be a sample of independent realizations of
the random pair $(X,Y)\in \X\times [0,1]$. Let $\left(\Fm\right)_{m\in \Mn^*}$
be a countable collection of models included in some countable family ${\mathcal{F}}\subset \{f : \X \mapsto [0,1] \ ; \ f\in
\L^2(\X)\}$ (which may depend on $X_1^n$). Consider some subcollection of models $(\Fm)_{m\in \Mn}$, where
$\Mn\subset \Mn^*$ may depend on $\Ll$, and some penalty function $\pen:\Mn \longrightarrow
\mathbb{R}_{+}$. Let $\ef$ (\ref{es}) be the corresponding penalized estimator
of the target function $f^*$. Take a family of weights $(x_m)_{m\in \Mn^*}$ such that 
\begin{equation} \label{sigma2}
\sum\limits_{m\in \Mn^*}e^{-x_m}\leq \Sigma <+\infty,
\end{equation}
with $\Sigma$ deterministic. Assume that assumptions ${\mathbf{H_1}}$, ${\mathbf{H_2}}$ and ${\mathbf{H_3}}$
hold.\\
 Let $\xi>0$. Hence, given some absolute constant $C>1$, there exist some
positive constants $K_1$ and $K_2$ such that, if for all $m\in \Mn$
$$\pen(m)\geqslant K_1\v_{m}^2+K_2c^2\frac{x_m}{n},$$ then, with probability
larger than $1-2\Sigma e^{-\xi}$, $$\lambda(f^*,\ef)\leq C \ \inf_{m\in
  \Mn}\left[\lambda(f^*,\Fm)+\pen(m)\right]+C'\ c^2 \frac{1+\xi}{n},$$ where
$\lambda(f^*,\Fm)=\inf_{f_m\in \Fm}\lambda(f^*,f_m)$ and the constant $C'$ only
depends on $C$. 
\end{theorem}

\begin{proof}
\ni The proof is highly similar to that of Theorem
\ref{thind}. The main differences are in the
conditioning and the fact that the collection of models $(\Fm)_{m\in \Mn}$ is
random. To remove these issues, all the bounds are computed uniformly on
$\Mn^*$ so that the probability of the set we finally obtain is
unconditional to $X_1^n$ since $\Sigma$ is deterministic. The inequalities are
obtained by the same techniques as the ones used for the proof of the results
on model selection on random models done by Gey and N\'ed\'elec in \cite{GeyNed05}. 

\medskip

\ni Let $m\in \Mn$ and $f_m\in \Fm$. Starting from (\ref{eqprinc}), we have 
\begin{eqnarray} \label{eqprinc2}
\lambda(f^*,\ef) & \leq & \lambda(f^*,f_m)+w_{\hat{m},m}(\ef)V_{\hat{m},m}+\pen(m)-\pen(\hat{m}),
\end{eqnarray}
where for all $m'$ and $M$ in $\Mn^*$, for all $f\in \mathcal{F}_{m'}$ and $f_M\in \mathcal{F}_M$,
$$w_{m',M}(f)=\left[\sqrt{l(f^*,f)}+\sqrt{\lambda(f^*,f_M)}\right]^2+(y_{m'}+y_M)^2,$$
$$V_{m',M}=\sup_{f\in
  {\mathcal{F}}_{m'}}\left[\frac{\bar{\g}_n(f_M)-\bar{\g}_n(f)}{w_{m',M}(f)}\right],$$
  with $y_{m'}\geqslant \v_{m'}$ and $y_M\geqslant \v_M$. The general
  principle is now exactly the same as in the proof of Theorem \ref{thind}
  despite the fact that we have to bound $V_{m',M}$ not only uniformly in
  $m'\in \Mn^*$, but also in $M\in \Mn^*$ in order to have an in-probability
  inequality that does not depend on $X_1^n$.

\smallskip

\ni Assumption ${\mathbf{H_2}}$ allows to give exactly the same upper bounds
(except that they depend on $X_1^n$ and that $y_{m'}$ is replaced by
$y_{m'}+y_M$) as (\ref{unifbound}) and (\ref{varbound}). 
By using the same techniques as in the proof of Theorem \ref{thind} and the same considerations as in \cite{GeyNed05}, we obtain that $$\E\left[V_{m',M} \ | \  X_1^n\right]\leq
 8\sqrt{10}\frac{\phi_{m'}(cy_{m'}+cy_M)}{(y_{m'}+y_M)^2}+\frac{c}{(y_{m'}+y_M)\sqrt{2n}}.$$
 Then, since $y_{m'}+y_M\geqslant y_{m'}\geqslant \v_{m'}$ and $\v_M>0$, we get
 by definition of $\v_{m'}$
 $$8\sqrt{10}\frac{\phi_{m'}(cy_{m'}+cy_M)}{(y_{m'}+y_M)^2}\leq
 8\sqrt{10}\frac{\phi(c\v_{m'})}{(y_{m'}+y_M)\v_{m'}}\leq
 8\sqrt{10}\frac{\v_{m'}+\v_M}{y_{m'}+y_M}.$$ So we have $$\E\left[V_{m',M} \ |
   \ X_1^n\right]\leq
 \frac{8\sqrt{10}(\v_{m'}+\v_M)+c(2n)^{-1/2}}{y_{m'}+y_M}.$$ Summing up over
 $m'\in \Mn^*$ and $M\in \Mn^*$, that leads by
 Rio's inequality, to
\begin{eqnarray*} 
V_{m',M} & \leq & \frac{1}{y_{m'}+y_M}\left(8\sqrt{10}\v_{m'}+\frac{c(2n)^{-1/2}}{2}+8\sqrt{10}\v_M+\frac{c(2n)^{-1/2}}{2}\right)\\
         & &+\sqrt{\frac{c^2+16(8\sqrt{10}(\v_{m'}+\v_M)+c(2n)^{-1/2})(y_{m'}+y_M)^{-1}}{2n(y_{m'}^2+y_M^2)}(x_{m'}+x_M+\xi)}\\
         & &+\frac{1}{y_{m'}^2+y_M^2}\left(\frac{x_{m'}+\xi/2}{n}+\frac{x_M+\xi/2}{n}\right)
\end{eqnarray*}
on a set $\Omega_{\xi}$ such that $P\left(\Omega_{\xi} \ | \ X_1^n\right)\geqslant 1-2\Sigma e^{-\xi}$. Then, since $\Sigma$ is deterministic, we get that $P(\Omega_{\xi})\geqslant 1-2\Sigma e^{-\xi}$.

\smallskip

\ni Hence, if we take for all $m'\in \Mn^*$
$$y_{m'}=2K\left[8\sqrt{10}\v_{m'}+\frac{c(2n)^{-1/2}}{2}+c\sqrt{\frac{x_{m'}+\xi/2}{n}}\right],$$
we obtain that, on $\Omega_{\xi}$, for all $m'$ and $M$ in $\Mn^*$,
$$V_{m',M}\leq
\frac{1}{K}\left[1+\sqrt{\frac{1}{2}\left(1+\frac{8}{K\sqrt{2}}\right)}+\frac{1}{2K\sqrt{2}}\right].$$
Finally the proof is achieved in the same way as the proof of Theorem \ref{thind}.
\end{proof}

\medskip

{\underline{\it Application to classification trees}}:

\smallskip

Let us consider the classification framework and the collection of models
$(\F_T)_{T\< T_{max}}$ obtained via the growing algorithm
in CART (see subsection \ref{subsec:prun}) as recalled in subsection \ref{subsec:proofM1}. Since the
growing and the pruning algorithms are made on the same sample $\Ll$, the conditions of Theorem \ref{thss} hold. Since $n=n_l$ is fixed, let us consider $\Mn^*$ as the set of all possible tree-structured partitions that can be constructed on the grid $X_1^n$, corresponding to trees having all possible splits in $\mathcal{S}$ and all possible forms without taking account of the response variable $Y$. So $\Mn^*$ depends only on the grid $X_1^n$ and is independent of the variables $(Y_1,\ldots,Y_n)$. Then $\{T\< T_{max}\} \subset \Mn^*$ and we are able to apply Theorem
\ref{thss}. Considering (\ref{egalite}), we take
$$d_n^2(f,g)=\frac{1}{n}\sum_{i=1}^n\left(f(X_i)-g(X_i)\right)^2,$$ corresponding with the distance $d$ given in Definition \ref{def:variance}. Using the margin condition {\bf MA(1)}, (\ref{module}) is also verified for $\lambda$ and $d_n$, and we have assumption ${\mathbf{H_2}}$ with $c^2=1/h$. Then, by Lemma \ref{locbd}, assumption ${\mathbf{H_3}}$ is checked with $\phi_T(x)=2x\sqrt{|\Tf|/n}$ and, in the same way as in the proof of Proposition \ref{pruM1}, $\v_T$ is taken as $\v_T=\sqrt{1/h}\sqrt{|\Tf|/n}$.\\ Finally, to choose a convenient family of
weights $(x_{_T})_{T\in \Mn^*}$, taking (see \cite{GeyNed05})
$$x_T=V\left(\theta +
  \log{\frac{n_1}{V}}\right)|\Tf|,$$ where $V$ is the VC-dimension of the set of splits ${\mathcal{S}}$ used to construct $T_{max}$ and $\theta >1$, we obtain
$$\Sigma_{\a}=\Sigma_{\theta}= \sum_{D\geqslant
  1}\exp{(-(\theta-1)DV)}<+\infty.$$ And we have Proposition \ref{pruM2}.

\subsection{Proof of Proposition \ref{FS}} \label{subsec:proofFS}

Proposition \ref{FS} is a direct application of the theorem obtained by Blanchard and Massart in \cite{Kol06rej}, reformulated for our purpose here: assume that we observe $N+n$ independent random variables with common
distribution $P$ depending on a parameter $f^*$ to be estimated. Suppose the
first $N$ observations $Z'=Z_1',\ldots,Z_N'$ are used to build some preliminary
collection of estimators $(\eft{m})_{m\in \Mn}$ and the remaining observations
$Z_1,\ldots, Z_n$ are used to select an estimator $\ef$ among this collection by
minimizing the empirical contrast as defined by (\ref{EC}) (with $(X,Y)$
replaced by $Z$). Hence, we have the following result.

\smallskip

\begin{th-a}[Blanchard and Massart \cite{Kol06rej}] \label{holdout}
~\\
Suppose that $\Mn$ is finite with cardinal $K$. Assume that there exists some continuous function
$w$ mapping $\R_+$ onto $\R_+$ such that $x\mapsto w(x)/x$ is nonincreasing, and which satisfies for all $\v>0$
\begin{eqnarray} \label{variance}
\sup_{\{f\in
  \F \ ; \ l(f^*,f)\leq
  \v^2\}}{\mathrm{Var}}\left[\g(f,Z)-\g(f^*,Z)\right]\leq w(\v). 
\end{eqnarray}
Then one has
for every $\theta \in (0,1)$ $$(1-\theta)\E\left[l(f^*,\ef) \ | \
  Z'\right]\leq (1+\theta)\inf_{m\in
  \Mn}l(f^*,\eft{m})+\delta_*^2\left(2\theta+(1+\log{(K)})(\frac{1}{3}+\frac{1}{\theta})\right),$$ where $l$ is defined by (\ref{loss}) and $\delta_{*}$ satisfies $\sqrt{n}\delta_{*}^2=w(\delta_{*})$.  
\end{th-a}

Taking $w(\v)=(1/\sqrt{h})\v$ for both methods M1 and M2, where $h$ is the margin, leads to proposition \ref{FS} with $$C=\frac{1+\theta}{1-\theta}, \ \ \ \ \ \ \ C_1=
\frac{\theta+3}{2\theta(1-\theta)}, \ \ \ \ \ \ C_2=C_1+\frac{\theta}{1-\theta}.$$ 

\subsection{Proof of Theorem \ref{thfin}} \label{subsec:proofthfin}

\ni We are now able to prove Theorem \ref{thfin} via propositions \ref{pruM1},
\ref{pruM2} and \ref{FS}. The beginning of the proof remains the same if $\ef$ is constructed
either via M1 or M2. So we just give the first step of the proof for the M1 method.\\ Actually,
since we have at most one model per dimension in the pruned subtree sequence,
it suffices to note that $K\leq n_1$. Then let $\a_0$ be the minimal constant
given by Proposition \ref{pruM1}. Hence, since for a given $\a>0$ $T_{\a}$
belongs to the sequence $(T_k)_{1\leq k\leq K}$, $$\E\left[l(f^*,\ef) \ | \ \Ll_1,\ \Ll_2\right]\leq C'' \
\underset{\a>\a_0}{\inf}l(f^*,\eft{T_{\a}})+C_1' \
h^{-1}\frac{\log{K}}{n_t}+h^{-1}\frac{C_2'}{n_t}.$$ Starting from this inequality, if $\ef$ is constructed via M1, by using Proposition
\ref{pruM1} with $\a=2\a_0$ and by taking the expectation according to
$\Ll_2$, we obtain Theorem \ref{thfin} with the appropriate constants. \\

\ni Yet, if $\ef$ is constructed via M2, we apply Proposition \ref{pruM2} with $\a = 2 \a_0 \a_{n_1,V}$ and, for each $\delta \in ]0;1[$, $\xi=\log{(2\Sigma_{\a}/\delta)}$. Then, we obtain Theorem \ref{thfin} with the appropriate constants. 


\end{document}